\pdfoutput=1

\documentclass[11pt]{article}

\usepackage[final]{acl}

\usepackage{times}
\usepackage{latexsym}

\usepackage[T1]{fontenc}

\usepackage[utf8]{inputenc}

\usepackage{microtype}

\usepackage{inconsolata}
\usepackage{footmisc}
\usepackage{microtype}
\usepackage{colortbl}
\usepackage{soul}
\usepackage{color, xcolor}
\usepackage{inconsolata}
\usepackage{amsmath}
\usepackage{amssymb}
\usepackage{mathtools}
\usepackage{amsthm}
\usepackage{bbm}
\usepackage{enumitem}
\usepackage{makecell}
\usepackage{blindtext}
\usepackage{graphicx}
\usepackage{capt-of}
\usepackage{booktabs}
\usepackage{varwidth}
\usepackage{multirow}
\usepackage{xspace,mfirstuc,tabulary}
\usepackage[capitalize,noabbrev]{cleveref}
\usepackage{booktabs} 
\usepackage{arydshln}
\usepackage{threeparttable}

\crefname{theorem}{Theorem}{}
\crefname{hypothesis}{Hyp.}{Hypotheses}
\crefname{assumption}{Assumption}{}
\crefname{prop}{Proposition}{}

\theoremstyle{plain}
\newtheorem{theorem}{Theorem}[section]

\newtheorem{definition}[theorem]{Definition}
\newtheorem{hypothesis}[theorem]{Hypothesis}
\newtheorem{proposition}[theorem]{Proposition}
\newtheorem{assumption}[theorem]{Assumption}

\newcommand{\defn}[1]{\textbf{#1}}
\newcommand{\boldY}{\boldsymbol{Y}}

\newcommand{\yy}{\boldsymbol{y}}
\newcommand{\ent}{\mathrm{H}}

\newcommand{\ours}{COIECD\xspace}
\definecolor{darkblue}{rgb}{0.0,0.0,0.5}
\definecolor{purple}{rgb}{0.5,0.0,0.5}
\newcommand{\db}[1]{\textcolor{darkblue}{\bf\scriptscriptstyle \selectfont \,(#1)}}
\newcommand{\ddel}[1]{\textcolor{purple}{\bf\scriptscriptstyle \selectfont \,(#1)}}

\newcommand{\ra}[1]{\renewcommand{\arraystretch}{#1}}
\definecolor{r}{rgb}{244, 193, 187}
\definecolor{g}{rgb}{131, 197, 183}

\title{Discerning and Resolving Knowledge Conflicts through Adaptive\\
 Decoding with Contextual Information-Entropy Constraint}

\author{
Xiaowei Yuan$^{1,2,3}$, Zhao Yang$^{1,2}$, Yequan Wang$^{*3}$, Shengping Liu
$^4$, Jun Zhao$^{1,2}$, Kang Liu\thanks{Corresponding authors}$^{1,2,5}$\\
$^1$The Laboratory of Cognition and Decision Intelligence for Complex Systems,\\
Institute of Automation, Chinese Academy of Sciences\\
$^2$School of Artificial Intelligence, University of Chinese Academy of Sciences\\
$^3$Beijing Academy of Artificial Intelligence, $^4$Unisound AI Technology Co., Ltd.\\
$^5$Shanghai Artificial Intelligence Laboratory\\
\texttt{
yuanxiaowei2022@ia.ac.cn,
\{zhao.yang, jzhao, kliu\}@nlpr.ia.ac.cn}\\
\texttt{tshwangyequan@gmail.com, liushengping@unisound.com}
}

\begin{document}
\maketitle

\begin{abstract}
Large language models internalize enormous \textit{parametric knowledge} during pre-training. Concurrently, realistic applications necessitate external \textit{contextual knowledge} to aid models on the underlying tasks. This raises a crucial dilemma known as \textit{knowledge conflicts}, where the contextual knowledge clashes with the parametric knowledge. 
However, existing decoding works are specialized in resolving knowledge conflicts and could inadvertently deteriorate performance in absence of conflicts.
In this paper, we propose an adaptive decoding method, termed as contextual information-entropy constraint decoding (COIECD), to discern whether the knowledge conflicts occur and resolve them. It can improve the model's faithfulness to conflicting context, and simultaneously maintain high performance among non-conflicting context. 
Our experiments show that COIECD exhibits strong performance and robustness over knowledge conflicts in realistic datasets. Code is available at \href{https://github.com/Stacy027/COIECD}{https://github.com/Stacy027/COIECD}.
\end{abstract}

\section{Introduction}

Characterized by the massive knowledge internalized into the parameters~\cite{PetroniRRLBWM19, GevaSBL21, RobertsRS20}, Large language models (LLMs) have pioneered numerous breakthroughs across various domains~\cite{vaswani2017attention, devlin2018bert, brown2020language, chung2022scaling, touvron2023llama}. 
Meanwhile, LLMs struggle with less popular factual knowledge~\cite{POPQA}, are fundamentally incapable of adapting over time~\cite{LazaridouKGALTG21,RealTimeQA} and prone to hallucinations~\cite{Shuster}. These challenges necessitate the incorporation of non-parametric knowledge sources, through retrieval~\cite{shi2023replug} or application of tools~\cite{Toolformer}.
However, it has given rise to a sharp dilemma: \textit{knowledge conflicts}, defined by \citet{LongprePCRD021}, where the non-parametric \textit{contextual knowledge} conflicts with internal \textit{parametric knowledge}.
Prior works~\cite{LongprePCRD021,ChenZC22,LiRZWLVYK23,prompting,Wangyike} have flagged that when confronting conflicts, larger models have a greater tendency to ignore the given context when it contradicts with model’s parametric knowledge. As shown in the Figure~\ref{Figure1}, due to the model's bias towards its parametric knowledge, it fails to ground its answer in the conflicting context.
\begin{figure}[t]
\begin{center}
\includegraphics[scale=0.45]{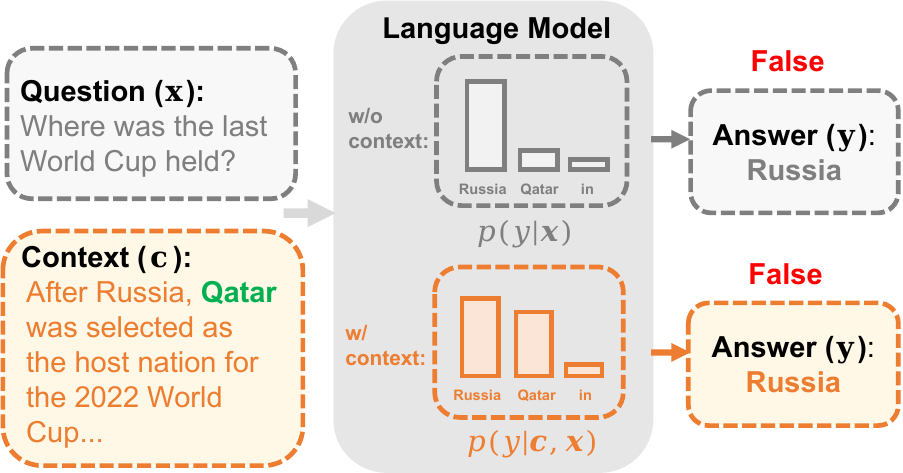} 
\caption{The illustration of knowledge conflict. Due to model's bias towards its outdated parametric knowledge, it fails to accurately ground answer in the latest context, which conflicts with the LM's knowledge.}
\label{Figure1}
\end{center}
\end{figure}

Early attempts on knowledge conflict-resolving methods resort to fine-tuning a small-scale model like T5~\cite{T5} by data augmentation, such as KAFT~\cite{LiRZWLVYK23} and DisentQA~\cite{NeemanAHCSA23}. Those fine-tuning methods bear the risk of undermining the intrinsic linguistic capabilities of the models~\cite{DBLP:journals/corr/abs-2310-05492}. 
Another line of works employ various decoding strategies during inference. For instance, 
Contrastive Decoding (CD)~\cite{cd,WangWLGYR23} leverages the discrepancy in contextual impact on the model's probability distribution of high-probability words for decoding.
Another representative method is Context-Aware Decoding (CAD)~\cite{cad}, which draws upon CD to amplifies the contextual distribution for all words.  
However, \textbf{existing decoding methods could inadvertently deteriorate performance in absence of conflicts}. 
As evidenced in the Figure~\ref{Figure2}, while these methods effectively mitigate over-reliance on parametric memory for knowledge conflicts, their performances deteriorate on the non-conflicting data derived from NaturalQuestions dataset. Typically, these methods generally work under the experimental scenario where all contexts are presumed to be inherently conflicting, without considering the presence or absence of conflicts in realistic scenario.
Thus, we posit the core question lies in: \textbf{how to discern knowledge conflicts between contexts and LLMs during inference}. 

\begin{figure}[t]
\begin{center}
\includegraphics[scale=0.35]{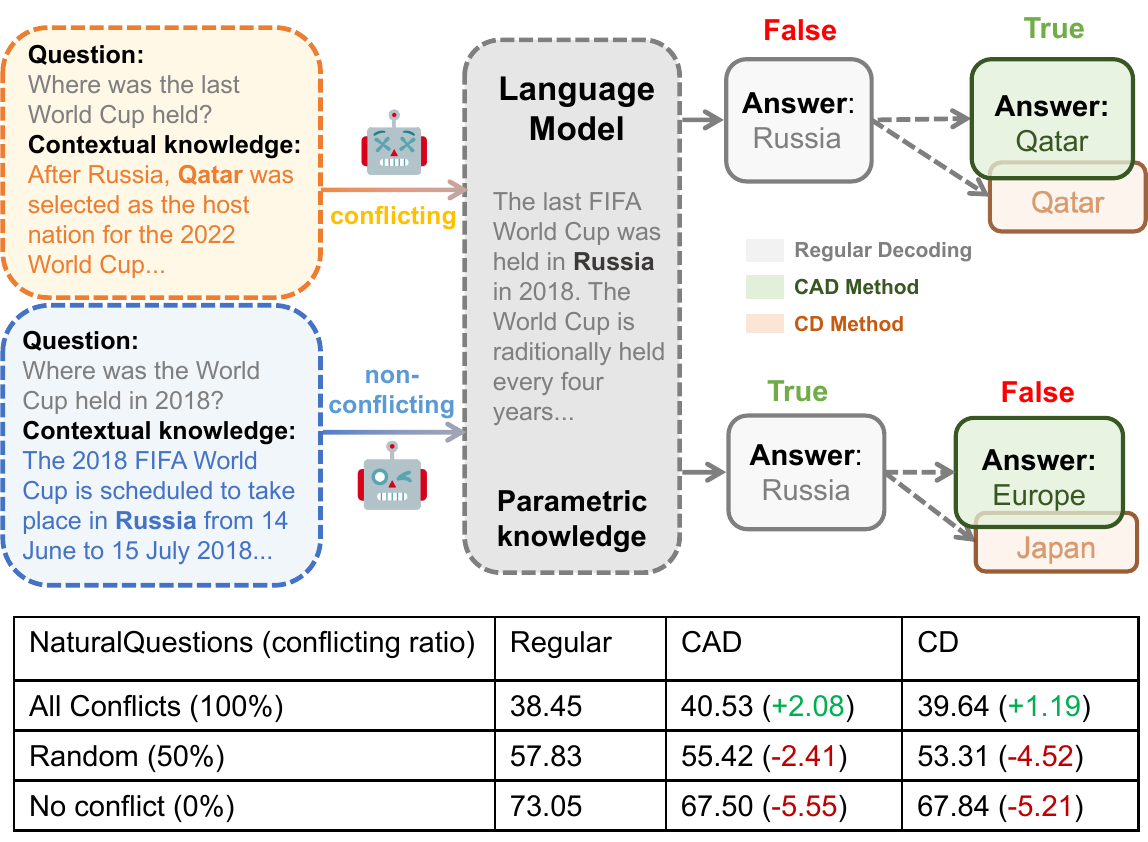} 
\caption{The illustration of conflicting and non-conflicting scenarios. Existing methods adeptly handle conflicts but struggle to address non-conflicting contexts. The table presented below illustrates the EM scores of existing conflict-solving methods and regular decoding method across diverse conflict ratio data. Numbers within brackets are the discrepancy between Regular and current method. More detailed analyses are in Appendix~\ref{appendix:analysis}.}
\label{Figure2}
\end{center}
\end{figure}

To this end, the paper proposes an adaptive decoding method, termed \textbf{CO}ntextual \textbf{I}nformation-\textbf{E}ntropy \textbf{C}onstraint \textbf{D}ecoding (\textbf{COIECD}), aimed at discerning knowledge conflicts and employing distinct strategies for conflicting and non-conflicting data.
Given the observations that LLMs tend to be well-calibrated~\cite{Kadavath} and their generations usually lie in a narrow and nearly flat entropy band~\cite{stableentropy}, we adopt an adaptive decoding strategy that only alleviates conflict when LLMs generate tokens violate an entropy-information constraint (band). To be specific, when discerning knowledge conflicts, it is important to consider whether LLMs have already aligned with contextual knowledge. If so, the entropy of contextual generation would not have a drastic change.
Therefore, we propose discerning the knowledge conflicts by measuring the changes of the distribution entropy at token level, and then employ tailored decoding strategies for conflicting and non-conflicting tokens.

We benchmark COIECD on several popular context-specific question answering (QA) datasets, including NaturalQuestions (NQ)~\cite{kwiatkowski2019natural}, SQuAD 1.1~\cite{squad}, StrategyQA~\cite{geva2021did}, and Counterfacts~\cite{LongprePCRD021}.
Over all tasks, COIECD achieves superior or competitive performance compared to the baselines, demonstrating the effectiveness and robustness of our method.

To summarize, the highlights of the paper are as follows:
\begin{itemize}
    \item This study presents a contextual information-entropy constraint to discern knowledge conflicts, between parametric knowledge in LLMs and non-parametric contextual knowledge. The constraint has proven effective in realistic datasets, which are characterized by the unpredictability of conflicts.
    \item The paper develops tailored decoding strategies to solve knowledge conflicts based on the contextual information-entropy constraint. Experimental results demonstrate that our method significantly augments the model's faithfulness to conflicting contexts and exhibits enhanced performance and robustness varying across diverse datasets and models.
\end{itemize}

\section{Related Work}
When presented with an external context with conflicting knowledge, 
prior works~\cite{LongprePCRD021,ChenZC22} have flagged that larger models have a greater tendency to ignore the conflicting context. 
Existing approaches for improving model’s faithfulness to the context, such as the prompting-based method~\cite{prompting}, is limited to specific instruction-finetuned LLMs and do not universally apply.
Other methods resort to fine-tuning a small-scale model like T5~\cite{T5} by counterfactual contexts, such as KAFT~\cite{LiRZWLVYK23} and DisentQA~\cite{NeemanAHCSA23}.
\citet{Wangyike} proposed an evaluation framework for simulating contextual knowledge conflicts and quantitatively evaluating to what extent LLMs achieve these goals.

Another line of works employ various decoding strategies during inference. 
SC~\cite{self-consistency} proposed the idea that a complex QA problem typically admits multiple different ways of thinking leading to its unique correct answer. It acts as a general enhanced decoding strategy.
CD~\cite{cd} adopted a contrastive object, which measures the discrepancy between two distributions to facilitate decoding. In addressing knowledge conflicts, this discrepancy is assessed based on the output probabilities with and without context. 
~\citet{dola} proposed contrastive layer decoding to enhance factuality, diverging from our focus.
Most similar to our work is the CAD~\cite{cad} method. It broadly amplifies the contextual distribution for all words without considering the presence of conflicting contexts, a limitation our work aims to address.

\section{Contextual Information-Entropy Constraint Decoding}

\paragraph{Discerning Conflicts (\S\ref{sec:constraint}).}
First, we argue that if a context has consistent knowledge with the model’s parameters, this context could be a natural generation of the model.\footnote{The assumption is empirically validated by a comparison of distribution entropy between conflicting and non-conflicting contexts, as detailed in the Appendix~\ref{appendix:assumption}.} It motives us to employ the theories of \textit{Stable Entropy Hypothesis}~\cite{stableentropy} and \textit{Locally Typical Set}~\cite{typicalsampling}\footnote{The detailed definitions of these concepts are provided in Appendix~\ref{properties}.} to measure whether there are unnatural tokens (conflicting knowledge) in the contexts, which demonstrate that natural-sounding language should ideally be constrained within a specific range. Based on these two theories, we introduce a novel decoding constraint termed the \textbf{contextual information-entropy constraint} which aims to identify the violation of token that results in less contextual generation attributed to knowledge conflicts, as shown in Figure \ref{fig:method}.


\paragraph{Resolving Conflicts (\S\ref{sec:method}).}
Then we implement tailored decoding strategies, which cater to tokens identified as either conflicting or non-conflicting. For non-conflicting tokens, model is expected to refer to both parametric and contextual knowledge.
For conflicting tokens, model should prioritize the contextual knowledge.
To this end, we calculate a contextual contrastive object~\cite{cd}, which represents the distribution discrepancy derived from the context. This object is then utilized to variably adjust token distributions in accordance with the contextual information-entropy constraint.

\begin{figure}[t]
\begin{center}
\includegraphics[scale=0.55]{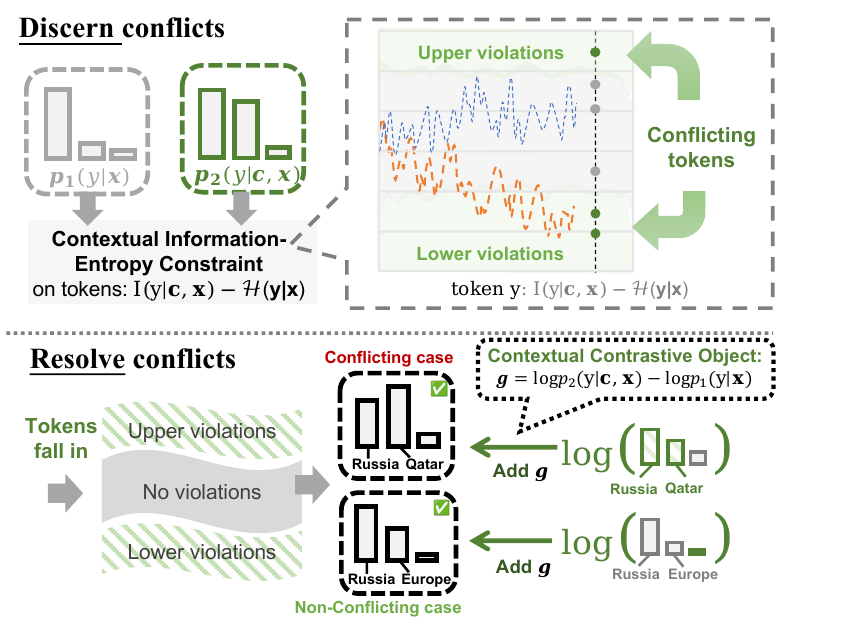} 
\caption{Above: Based on the contextual information-entropy constraint, tokens that fall into either the lower or upper violation zone of the constraint are typically associated with conflicts. Below: Distinct decoding strategies are employed for conflicting and non-conflicting tokens.}
\label{fig:method}
\end{center}
\end{figure}

\subsection{Contextual Information-Entropy Constraint} \label{sec:constraint}
We assume that if the contextual knowledge aligns with model's parametric knowledge, then the context can be a coherent and natural generation of model in some way. In this setting, the characteristics of natural language generation hold true for the non-conflicting contexts.
Given the observations that LLMs tend to be well-calibrated~\cite{Kadavath} and their generations usually lie in a narrow and nearly flat entropy band~\cite{stableentropy}, we craft a contextual constraint to measure the changes of the distribution entropy and token information, using it as an indicator to discern knowledge conflicts on token-grained level.
We define the entropy of the genertated token $\boldsymbol{y}_t$ by given the question $\boldsymbol{x}$ and generated history $\boldsymbol{y}_{<t}$ following~\citet{stableentropy} as
\begin{equation}
    \mathcal{H}(\boldsymbol{y}_t|\boldsymbol{x},\boldsymbol{y}_{<t}) = \underset{y_t \sim p(\cdot| \boldsymbol{y}_{<t})}{\mathbb{E}}  -\log p(y_t|\boldsymbol{x},\boldsymbol{y}_{<t}) 
\end{equation}
For brevity, we use $\mathcal{H}_{1}(\boldsymbol{y_t})$ to represent the entropy of conditional distribution over question $\boldsymbol{x}$ and generated history $\boldsymbol{y}_{<t}$, and $\mathcal{H}_{2}(\boldsymbol{y_t})$ denotes the entropy conditioning by $\boldsymbol{x}$, $\boldsymbol{y}_{<t}$, and assumed generation $\boldsymbol{c}$.
\begin{align}
    \mathcal{H}_{1}(\boldsymbol{y_t}) &= \mathcal{H}(\boldsymbol{y}_t|\boldsymbol{x},\boldsymbol{y}_{<t})\\
    \mathcal{H}_{2}(\boldsymbol{y_t}) &= \mathcal{H}(\boldsymbol{y}_t|\boldsymbol{x},\boldsymbol{c},\boldsymbol{y}_{<t}) 
\end{align}

The \textit{Stable Entropy Hypothesis}~\cite{stableentropy} proposes that natural language generations usually lie in a narrow and flat entropy band.
When we posit that a non-conflicting context can arise as a natural generation of model, the entropy shift should adhere to the bound. Deviations from it may indicate a potential conflicting context. 
In such instances, it becomes crucial to precisely identify which specific tokens, reflecting the conflicts, are likely to cause the model to exceed its entropy bound during generation. 
To address this, we utilize the \textit{Locally Typical Set}~\cite{typicalsampling} to discern tokens by the following bound on information-entropy shift. The proofs are detailed in Appendix~\ref{derivations}.
\begin{proposition}[Bound on information-entropy shift]
The information content of a random variable is quantified as its negative log-probability~\cite{typicalsampling}.
Let the information content of token $y_t$ be  $I(y_t)=-\log p(y_t\mid \boldsymbol{x},\boldsymbol{c},\boldsymbol{y}_{<t})$, and we define a \textbf{information-entropy shift} as: $I(y_t) - \mathcal{H}_{1}(\boldsymbol{y_t}) $. The following bound holds for a constant $\gamma > 0$:
\begin{equation} 
    \big | I(y_t) - \mathcal{H}_{1}(\boldsymbol{y_t}) \big | < \gamma 
\end{equation}
\end{proposition}
\noindent In words, the information-entropy shift can be bounded by some constant denoted as $\gamma$. That means, if the shift of a token adheres to this constraint, we can view it as a plausible candidate of non-conflicting contextual generation. Conversely, any violation of token indicate the potential conflicts with a high probability.

To formalize the bound into constraint of decoding, we follow a popular constraint paradigm in decoding techniques such as nucleus sampling~\cite{holtzman2019curious} and CD~\cite{cd}. We employ the \textit{softmax} function to normalize the information-entropy shift into distribution:
\begin{equation} \label{def:shift}
    p_{\delta}(y_t) = \mathrm{softmax}(I(y_t) - \mathcal{H}_{1}(\boldsymbol{y_t}))
\end{equation}
Then we have an upper bound $u_{p_{\delta}}$ and a lower bound $l_{p_{\delta}}$ to constrain the vocabulary subset when decoding as:
\begin{align}
&u_{p_{\delta}} = \lambda \max_w p_{\delta}(w) \label{upper}\\
&l_{p_{\delta}} = 
\begin{cases} 
  l'_{p_{\delta}}  & \text{if } \sum \mathbb{I}(p_{\delta}(y_t) < l'_{p_{\delta}}) > 1, \\
  0 & \text{otherwise}.
\end{cases} \label{lower}
 \\& \text{where}\quad l'_{p_{\delta}} = \frac{1}{\lambda} \min_w p_{\delta}(w) \nonumber
\end{align}
Here $\lambda$ is a scaling factor in $(0, 1]$ and $\mathbb{I}$ is an indicator function. 
Eq.~\ref{lower} implies that the lower-bound probability $l_{p_{\delta}}$ takes the value in cases where multiple tokens exhibit probabilities $p_{\delta}(y_t)$ falling below the $l'_{p_{\delta}}$.
Otherwise, \( l_{p_{\delta}} \) is set to 0, indicating that only a solitary token violates the lower bound. It reflects model's high confidence with the absence of conflict for that token.
Based on the bounds, the constraint subset $\mathcal{C}(\yy_{<t}) \subseteq \mathcal{V}$ is as follows:
\begin{equation}
    \mathcal{C}(\yy_{<t}) = \{y \in \mathcal{V}:  l_{p_{\delta}}
\leq p_{\delta}(y_t) \leq u_{p_{\delta}} \}
\end{equation}




\subsection{Adaptive Decoding} \label{sec:method}
Before employing distinct decoding strategies for the conflicting and non-conflicting tokens, initially, we define that
\begin{align}
    p_{1}(y_t) &= p(y_t|\boldsymbol{x},\yy_{<t}) \\
    p_{2}(y_t) &= p(y_t|\boldsymbol{x},\boldsymbol{c},\yy_{<t})
\end{align}
Here the parametric knowledge is factored out from the model’s output distribution as $p_1$, in accordance with~\citet{cad}.
The output distribution $p_2$ that incorporates context can be interpreted as context-aware knowledge, which integrates knowledge from both parameters and context.
Then a contextual contrastive object $g$~\cite{cd} is calculated to quantify the divergence between $p_1$ and $p_2$:
\begin{equation}
    g(y_t) = \log p_2(y_t) - \log p_1(y_t)
\end{equation}
which aims to refine the discrepancy brought by the context.
It assumes that $p_1$ has a stronger tendency to produce the outputs that adhere to parametric knowledge of the model. 
The $g$ is to factor out the model's inherent memory and favor the contextual knowledge. 

Based on $g$, the decoding strategies are differentiated for tokens distinguished by the proposed contextual information-entropy constraint.
For conflicting tokens, model is expected to prioritize contextual knowledge. To facilitate this, $g$ is strategically employed to reinforce context-aware knowledge $p_2$. For non-conflicting tokens, the model is encouraged to lean more heavily on parametric knowledge, rather than depending exclusively on context. This strategy stems from on the recognition of the potential limitations in contextual knowledge, which may not be comprehensive to fully address the query.
Therefore, this paper emphasizes the importance of parametric knowledge $p_1$, while still considering contextual factors. To achieve this, the $g$ is incorporated with it. 
Overall, the contextual information-entropy constraint is utilized with $g$ on the output distribution $\pi$ as:
\begin{align}
&\quad \log\pi(y_t \mid\boldsymbol{x},\boldsymbol{c},\yy_{<t}) \\
&=
\begin{cases}
\log p_1(y_t) +  \alpha \cdot  g(y_t) & \text{if } y_t \in \mathcal{C}(\yy_{<t}), \\
\log p_2(y_t) +  \alpha \cdot g(y_t) & \text{otherwise} .
\end{cases}\nonumber
\end{align}\label{eq:g(t)}
where $\alpha$ is a scaling weight to control the contextual impact.
The final decoding strategy can be formalized as:
\begin{equation}
    y_t \sim \mathrm{softmax}[\log\pi(y_t \mid\boldsymbol{x},\boldsymbol{c},\yy_{<t})] 
\end{equation}
In this way, COIECD strikes a balance between the two sources of knowledge to achieve a more effective and holistic decoding strategy.

\section{Experiments}

\begin{table*}[t]
\tiny
\centering
\resizebox{0.85\textwidth}{!}{
\begin{threeparttable}
\begin{tabular}{ccllllll}
\toprule
& &  \multicolumn{2}{c}{\textbf{LLaMA2-13B}} & \multicolumn{2}{c}{\textbf{OPT-6.7B}} & \multicolumn{2}{c}{\textbf{FLAN-T5-3B}}  \\
\cmidrule(lr){3-4}
\cmidrule(lr){5-6}
\cmidrule(lr){7-8}
\textbf{Datasets} & \textbf{Decoding} & EM & F1 & EM & F1& EM & F1 \\
\midrule
\multirow{5}{*}{ NQ } &  Regular & \cellcolor[gray]{0.9}46.48 &\cellcolor[gray]{0.9}61.51 &\cellcolor[gray]{0.9}19.74 &\cellcolor[gray]{0.9}26.25 &\cellcolor[gray]{0.9}46.00 &\cellcolor[gray]{0.9}62.78 \\
& \text{SC} & 46.66$\ddel{+0.18}$ & 61.76$\ddel{+0.25}$ & 24.24$\ddel{+4.50}$ & 29.78$\ddel{+3.53}$ & 46.14$\ddel{+0.14}$ & 62.51$\db{-0.27}$ \\
& \text{CD} & 46.19$\db{-0.29}$ & 61.97$\ddel{+0.46}$ & 22.90$\ddel{+3.16}$ & 34.48$\ddel{+8.23}$ & 37.62$\db{-8.38}$ & 55.47$\db{-7.31}$ \\
& \text{CAD} & 46.79$\ddel{+0.31}$ & 62.29$\ddel{+0.78}$ & 29.15$\ddel{+9.41}$ & 40.16$\ddel{+13.91}$ & 38.91$\db{-7.09}$ & 57.77$\db{-5.01}$ \\
& \ours & \textbf{47.42}$\ddel{+0.94}$ & \textbf{62.89}$\ddel{+1.38}$ & \textbf{30.07}$\ddel{+10.33}$ & \textbf{40.77}$\ddel{+14.52}$ & \textbf{48.84}$\ddel{+2.84}$ & \textbf{64.45}$\ddel{+1.67}$ \\
\midrule
\multirow{5}{*}{ SQuAD } & Regular &\cellcolor[gray]{0.9}54.46 &\cellcolor[gray]{0.9}68.92 &\cellcolor[gray]{0.9}21.49 &\cellcolor[gray]{0.9}28.50  &\cellcolor[gray]{0.9}71.20 &\cellcolor[gray]{0.9}83.53\\
& \text{SC} & 54.55$\ddel{+0.09}$ & 68.85$\db{-0.07}$ & 23.64$\ddel{+2.15}$ & 30.97$\ddel{+2.47}$ & 70.90$\db{-0.30}$ & 83.28$\db{-0.25}$ \\
& \text{CD} & 53.89$\db{-0.57}$ & 68.04$\db{-0.88}$ & 26.35$\ddel{+4.86}$ & 37.90$\ddel{+9.40}$  & 71.25$\ddel{+0.05}$ & 83.10$\db{-0.43}$\\
& \text{CAD} & 56.46$\ddel{+2.00}$ & 70.52$\ddel{+1.60}$ & 29.46$\ddel{+7.97}$ & 40.31$\ddel{+11.81}$  & 68.62$\db{-2.58}$ & 81.88$\db{-1.65}$\\ 
& \ours & \textbf{57.10}$\ddel{+2.64}$& \textbf{70.86}$\ddel{+1.94}$ & \textbf{29.93}$\ddel{+8.44}$ & \textbf{40.47}$\ddel{+11.97}$ & \textbf{73.84}$\ddel{+2.64}$ & \textbf{84.99}$\ddel{+1.46}$ \\
\midrule
\multirow{5}{*}{ StrategyQA } & Regular & \cellcolor[gray]{0.9}81.09 & \cellcolor[gray]{0.9}81.09 & \cellcolor[gray]{0.9}47.51 & \cellcolor[gray]{0.9}47.51  &\cellcolor[gray]{0.9}87.07 &\cellcolor[gray]{0.9}87.07 \\
& \text{SC} & 81.05$\db{-0.04}$ & 81.05$\db{-0.04}$  & 46.64$\db{-0.87}$ & 46.64$\db{-0.87}$ &  86.81$\db{-0.26}$  & 86.81$\db{-0.26}$ \\
& \text{CD} & 83.58$\ddel{+2.49}$ & 83.58$\ddel{+2.49}$ & 46.99$\db{-0.52}$ & 46.99$\db{-0.52}$ & \textbf{89.34}$\ddel{+2.27}$ & \textbf{89.34}$\ddel{+2.27}$ \\
& \text{CAD} & 85.50$\ddel{+4.41}$ & 85.50$\ddel{+4.41}$ & 53.10$\ddel{+5.59}$ & 53.10$\ddel{+5.59}$ & 88.69$\ddel{+1.62}$ & 88.69$\ddel{+1.62}$ \\ 
& \ours & \textbf{85.76}$\ddel{+4.67}$ &\textbf{85.76}$\ddel{+4.67}$  & \textbf{53.84}$\ddel{+6.33}$ & \textbf{53.84}$\ddel{+6.33}$ & 88.78$\ddel{+1.71}$ & 88.78$\ddel{+1.71}$ \\
\midrule
\midrule
\multirow{5}{*}{ Counterfacts }  & Regular & \cellcolor[gray]{0.9}61.67 &\cellcolor[gray]{0.9}62.63 &\cellcolor[gray]{0.9}18.15 &\cellcolor[gray]{0.9}19.38 &\cellcolor[gray]{0.9}74.56 &\cellcolor[gray]{0.9}75.73 \\
& \text{SC} & 61.76$\ddel{+0.09}$ & 62.76$\ddel{+0.13}$  & 21.40$\ddel{+3.25}$ & 22.62$\ddel{+3.24}$ & 74.58$\ddel{+0.02}$ & 75.64$\db{-0.09}$ \\
& \text{CD} & 67.96$\ddel{+6.29}$& 69.16$\ddel{+6.53}$ & 38.16$\ddel{+20.01}$& 42.78$\ddel{+23.40}$ & 74.76$\ddel{+0.20}$ & 77.30$\ddel{+1.57}$ \\
& \text{CAD} & \textbf{68.76}$\ddel{+7.09}$ & \textbf{71.20}$\ddel{+8.57}$ & \textbf{40.10}$\ddel{+21.95}$ & \textbf{45.29}$\ddel{+25.91}$ & 68.23$\db{-6.33}$ & 74.17$\db{-1.56}$  \\ 
&  \ours & 68.30$\ddel{+6.63}$ & 69.33$\ddel{+6.70}$ & 37.35$\ddel{+19.20}$ & 43.45$\ddel{+24.07}$ & \textbf{77.60}$\ddel{+3.04}$ & \textbf{78.97}$\ddel{+3.24}$  \\
\bottomrule
\end{tabular}
 \begin{tablenotes}
        \footnotesize
        \item[*] We reproduce all baseline methods and report our corresponding results.
 \end{tablenotes}
  \end{threeparttable}
}
\caption{\textbf{Totally, \textit{COIECD} achieves stable optimal performance than baselines.} \textit{Regular}: Regular decoding, \textit{SC}: Self-consistency, \textit{CD}: Contrastive decoding, \textit{CAD}: context-aware decoding. The best scores compared with \textit{Regular} are boldfaced. Numbers within brackets are the discrepancy between \textit{Regular} and current method. The outcomes for models of various sizes are detailed in Table~\ref{tbl:llama}-\ref{tbl:flant5}.
} 
\label{tbl:main}
\end{table*}

\begin{table*}[t]
\tiny
\centering
\resizebox{0.8\textwidth}{!}{
\begin{threeparttable}
\begin{tabular}{cccllllll}
\toprule
& & &  \multicolumn{2}{c}{\textbf{LLaMA2-13B}} & \multicolumn{2}{c}{\textbf{OPT-6.7B}} & \multicolumn{2}{c}{\textbf{FLAN-T5-3B}}  \\
\cmidrule(lr){4-5}
\cmidrule(lr){6-7}
\cmidrule(lr){8-9}
\multicolumn{2}{c}{\textbf{Datasets}} & \textbf{Decoding} & EM & F1 & EM & F1& EM & F1 \\
\midrule
\multirow{10}{*}{ NQ } & \multirow{5}{*}{ Conf. } & Regular &\cellcolor[gray]{0.9}38.45 &\cellcolor[gray]{0.9}54.37 &\cellcolor[gray]{0.9}19.79 &\cellcolor[gray]{0.9}26.24 &\cellcolor[gray]{0.9}45.16 &\cellcolor[gray]{0.9}61.44 \\
& & \text{SC} & 38.65$\ddel{+0.20}$ & 54.64$\ddel{+0.27}$ & 24.26$\ddel{+4.47}$ & 29.75$\ddel{+3.51}$ & 45.22$\ddel{+0.06}$ & 61.02$\db{-0.42}$ \\
& & \text{CD} & 39.64$\ddel{+1.19}$ & 56.50$\ddel{+2.13}$ & 22.96$\ddel{+3.17}$ & 34.54$\ddel{+8.30}$ & 38.29$\db{-6.87}$ & 55.97$\db{-5.47}$ \\
& & \text{CAD} & \textbf{40.53}$\ddel{+2.08}$ & \textbf{57.15}$\ddel{+2.78}$ & 29.21$\ddel{+9.42}$ & 40.19$\ddel{+13.95}$ & 39.34$\db{-5.82}$ & 58.25$\db{-3.19}$ \\
& & \ours & 39.88$\ddel{+1.43}$ & 56.59$\ddel{+2.22}$ & \textbf{30.13}$\ddel{+10.34}$ & \textbf{40.78}$\ddel{+14.54}$ & \textbf{48.36}$\ddel{+3.20}$ & \textbf{63.98}$\ddel{+2.54}$ \\
\cmidrule(lr){2-9}
& \multirow{5}{*}{ \makecell{Non-\\Conf.} } & Regular &\cellcolor[gray]{0.9}73.05 &\cellcolor[gray]{0.9}85.15 & \cellcolor[gray]{0.9}12.40 & \cellcolor[gray]{0.9}27.03 &\cellcolor[gray]{0.9}52.20 &\cellcolor[gray]{0.9}72.65 \\
& & \text{SC} & \textbf{73.16}$\ddel{+0.11}$ & \textbf{85.30}$\ddel{+0.15}$ & \textbf{21.79}$\ddel{+9.39}$ & 34.07$\ddel{+7.04}$& 52.26$\ddel{+0.06}$ & \textbf{73.49}$\ddel{+0.84}$ \\
& & \text{CD} & 67.84$\db{-5.21}$ & 80.06$\db{-5.09}$ & 12.51$\ddel{+0.11}$ & 25.23$\db{-1.80}$ & 33.40$\db{-18.80}$ & 52.33$\db{-20.32}$ \\
& & \text{CAD} & 67.50$\db{-5.55}$ & 79.19$\db{-5.96}$ & 20.83$\ddel{+8.43}$ & \textbf{36.05}$\ddel{+9.02}$ & 35.68$\db{-16.52}$ & 54.22$\db{-18.43}$ \\ 
& & \ours & 72.37$\db{-0.68}$ & 83.75$\db{-1.40}$ & 19.01$\ddel{+6.61}$ & 35.82$\ddel{+8.79}$ & \textbf{52.42}$\ddel{+0.22}$ & 67.87$\db{-4.78}$ \\
\midrule
\multirow{10}{*}{ SQuAD } & \multirow{5}{*}{ Conf. } & Regular & \cellcolor[gray]{0.9}48.78 & \cellcolor[gray]{0.9}64.34 & \cellcolor[gray]{0.9}21.49 & \cellcolor[gray]{0.9}28.50 &\cellcolor[gray]{0.9}70.51 &  \cellcolor[gray]{0.9}83.09\\
& & \text{SC} & 48.87$\ddel{+0.09}$ & 64.24$\db{-0.10}$ & 23.14$\ddel{+1.65}$ & 30.18$\ddel{+1.68}$ & 70.25$\db{-0.26}$ & 82.84$\db{-0.25}$ \\
& & \text{CD} & 50.68$\ddel{+1.90}$ & 66.01$\ddel{+1.67}$ & 26.33$\ddel{+4.84}$ & 37.61$\ddel{+9.11}$ & 71.31$\ddel{+0.80}$ & 83.17$\ddel{+0.08}$ \\
& & \text{CAD} & 51.64$\ddel{+2.86}$ & \textbf{67.09}$\ddel{+2.75}$ & 29.32$\ddel{+7.83}$ & 39.97$\ddel{+11.47}$ & 68.64$\db{-1.87}$ & 81.92$\db{-1.17}$ \\ 
& & \ours & \textbf{51.95}$\ddel{+3.17}$ & 66.91$\ddel{+2.57}$& \textbf{29.78}$\ddel{+8.29}$ & \textbf{40.13}$\ddel{+11.63}$ & \textbf{73.51}$\ddel{+3.00}$ & \textbf{84.76}$\ddel{+1.67}$ \\
\cmidrule(lr){2-9}
& \multirow{5}{*}{ \makecell{Non-\\Conf.} } & Regular & \cellcolor[gray]{0.9}80.50 & \cellcolor[gray]{0.9}89.88 & \cellcolor[gray]{0.9}35.62 & \cellcolor[gray]{0.9}57.10 &\cellcolor[gray]{0.9}79.56 &\cellcolor[gray]{0.9}88.81 \\
& & \text{SC} & 80.57$\ddel{+0.07}$ & \textbf{89.96}$\ddel{+0.08}$ & \textbf{36.59}$\ddel{+0.97}$ & \textbf{56.04}$\db{-1.06}$ & \textbf{78.67}$\db{-0.89}$ & \textbf{88.61}$\db{-0.20}$ \\
& & \text{CD} & 68.64$\db{-11.86}$ & 77.35$\db{-12.53}$ & 26.90$\db{-8.72}$ & 49.05$\db{-8.05}$ & 70.60$\db{-8.96}$ & 82.26$\db{-6.55}$ \\
& & \text{CAD} & 78.53$\db{-1.97}$ & 86.19$\db{-3.69}$ & 34.93$\db{-0.69}$ & 53.44$\db{-3.66}$ & 68.37$\db{-11.19}$ & 81.42$\db{-7.39}$ \\
& & \ours & \textbf{80.69}$\ddel{+0.19}$ & 88.93$\db{-0.95}$ & 35.69$\ddel{+0.07}$ & 53.71$\db{-3.39}$ & 77.78$\db{-1.78}$ & 87.76$\db{-1.05}$ \\
\midrule
\multirow{10}{*}{ StrategyQA } 
& \multirow{5}{*}{ Conf. } & Regular & \cellcolor[gray]{0.9}57.36 & \cellcolor[gray]{0.9}57.36 & \cellcolor[gray]{0.9}47.86 & \cellcolor[gray]{0.9}47.86 &\cellcolor[gray]{0.9}69.41 &\cellcolor[gray]{0.9}69.41 \\
& & \text{SC} & 57.59$\ddel{+0.23}$ & 57.59$\ddel{+0.23}$ & 47.03$\db{-0.83}$ & 47.03$\db{-0.83}$ & 68.80$\db{-0.61}$ & 68.80$\db{-0.61}$ \\
& & \text{CD} & \textbf{81.15}$\ddel{+23.79}$ & \textbf{81.15}$\ddel{+23.79}$  & 47.26$\db{-0.60}$ & 47.26$\db{-0.60}$ & \textbf{85.96}$\ddel{+16.55}$ & \textbf{85.96}$\ddel{+16.55}$  \\
& & \text{CAD} & 77.31$\ddel{+19.95}$ & 77.31$\ddel{+19.95}$ & 54.21$\ddel{+6.35}$ & 54.21$\ddel{+6.35}$ & 77.39$\ddel{+7.98}$ & 77.39$\ddel{+7.98}$  \\
& & \ours & 80.29$\ddel{+22.93}$ & 80.29$\ddel{+22.93}$ & \textbf{54.90}$\ddel{+7.04}$ & \textbf{54.90}$\ddel{+7.04}$ & 77.36$\ddel{+7.95}$ & 77.36$\ddel{+7.95}$  \\
\cmidrule(lr){2-9}
& \multirow{5}{*}{ \makecell{Non-\\Conf.} } & Regular & \cellcolor[gray]{0.9}96.54 & \cellcolor[gray]{0.9}96.54 & \cellcolor[gray]{0.9}40.87 & \cellcolor[gray]{0.9}40.87 &\cellcolor[gray]{0.9}97.06 &\cellcolor[gray]{0.9}97.06 \\
& & \text{SC} & \textbf{96.47}$\db{-0.07}$ & \textbf{96.47}$\db{-0.07}$ & 39.13$\db{-1.74}$ & 39.13$\db{-1.74}$ & \textbf{96.69}$\db{-0.07}$ & \textbf{96.69}$\db{-0.07}$  \\
& & \text{CD} & 85.16$\db{-11.38}$ & 85.16$\db{-11.38}$ & \textbf{41.71}$\ddel{+0.84}$ & \textbf{41.71}$\ddel{+0.84}$ & 91.26$\db{-5.80}$ & 91.26$\db{-5.80}$ \\
& & \text{CAD} & 89.33$\db{-7.21}$ & 89.33$\db{-7.21}$ & 32.17$\db{-8.70}$ & 32.17$\db{-8.70}$ & 95.08$\db{-1.98}$ & 95.08$\db{-1.98}$ \\
& & \ours & 90.80$\db{-5.74}$  & 90.80$\db{-5.74}$ & 33.91$\db{-6.96}$ & 33.91$\db{-6.96}$& 95.22$\db{-1.84}$ & 95.22$\db{-1.84}$ \\
\bottomrule
\end{tabular}
  \end{threeparttable}
}
\caption{We use the posteriori judgement of the parametric knowledge in LLMs~\cite{DBLP:conf/emnlp/WangLSL23} to identify and analyze conflicts within the datasets. On Non-Conf. data, \textit{COIECD} consistently outperforms other conflict-solving methods in terms of both EM and F1, and outperforms the \textit{Regular} and \textit{SC} on Conf. data.} 
\label{tbl:sub-main}
\end{table*}

\subsection{Experimental Setup}
\paragraph{Datasets.}  
We experiment with several public QA datasets, including NaturalQuestions~\cite{kwiatkowski2019natural}, SQuAD 1.1~\cite{squad} and StrategyQA~\cite{geva2021did}.
Unlike prior research where all data consists of synthetic conflicts, we adopt the original datasets and view them as hybrid datasets consisting of both conflicting (Conf.) and non-conflicting (Non-Conf.) data. It can stimulate the unpredictability of conflict occurrences in a realistic setting. 
Then we adopt the posteriori judgement of the parametric knowledge in LLMs \cite{DBLP:conf/emnlp/WangLSL23} to identify the knowledge conflicts within the datasets in Sec~\ref{sec:conf-nonconf}. 

Furthermore, we also incorporate the Counterfacts dataset~\cite{LongprePCRD021} to facilitate a more comprehensive analysis. 
Counterfacts exclusively consists of synthetic conflicting data, where all the original answers are replaced with other plausible entities in the contexts. The brief introductions and statistic for each dataset are provided in Appendix~\ref{appendix:data}.
We apply the prompt instruction following~\citet{ren2023investigating} to assess the QA abilities for all models. 
 
\paragraph{Used LLMs.} Our experiments are conducted on pre-trained language models, including auto-regressive models: the LLaMA2 models (7B, 13B parameters) \cite{touvron2023llama}, OPT models (6.7B, 13B parameters) \cite{zhang2022opt} and the encoder-decoder language model: FLAN-T5 (3B, 11B parameters) \cite{chung2022scaling}. 
The experimental results feature a representative outcome for a single size in each model. Additional results, including a comparative analysis of GPT-3.5 and GPT-4 performance on these datasets, are detailed in Appendix~\ref{appendix:more}.

\paragraph{Baselines.} \label{baselines}
We adopt four decoding methods as baselines: Regular Decoding, Self-Consistency (SC)~\cite{self-consistency}, Contrastive Decoding (CD)~\cite{cd} and Context-Aware Decoding (CAD)~\cite{cad}.\footnote{For the issue of knowledge conflicts, CD adopts the object of difference between the output likelihood when inputs are presented with and without context. More detailed comparisons of those methods are described in the Appendix~\ref{appendix:methods}.} CD and CAD are specialized in resolving knowledge conflicts, while SC is a general decoding strategy to strengthen the model performance.
Regular Decoding employs a standard, greedy strategy, integrating both question and context as inputs. For SC, which necessitates multiple samples per question, 40 outputs are sampled with temperature $t=0.5$, in according with~\citet{self-consistency}. For other methods, the temperature $t=0$ following prior works. All the decoding methods are evaluated in a zero-shot setting. The values of $\lambda$ and $\alpha$ are set to 0.25 and 1, respectively. Detailed analyses sampling strategies are provided in Appendices~\ref{appendix:strategy}.


\paragraph{Metrics.} Following previous works \cite{ChenFWB17, IzacardG21, sun2023beamsearchqa}, we use the Exact Match (EM) and F1 scores for evaluating the QA performance of LLMs. For the binary classification in StrategyQA, the accuracy is used as the metric.

\subsection{Overall Performance} \label{exp.1}
Table \ref{tbl:main} presents the results on the QA datasets. 
Totally, COIECD exhibits consistent improvements over all baseline comparisons.
The SC method yields results akin to the Regular with a slight increase.
The performance of conflict-solving methods, namely CD and CAD, varies across models and datasets, showing inconsistent variations when compared to Regular.
On the contrary, COIECD consistently achieves improvements in realistic datasets (NQ, SQuAD and StrategyQA) and maintains competitive performance in the synthetic Counterfacts dataset.
The results conclusively demonstrate the consistent effectiveness and adaptability of COIECD across various datasets in different conflict scenarios. 

The results on the Counterfacts dataset reveal that most methods exhibit performance enhancement. Upon closer examination, it becomes evident that the CAD's advantages are primarily evident in counterfactual scenarios, outperforming other methods except FLAN-T5.
Nonetheless, COIECD still demonstrates superior robustness, maintaining competitive performance across various models.
\subsection{Performance on Conf. \& Non-Conf. data.} \label{sec:conf-nonconf}

As shown in the Table~\ref{tbl:sub-main}, since the CD and CAD specialize in resolving knowledge conflicts, they can handle the Conf. data well. However, in the Non-Conf. dataset, both of them demonstrate a significant decrease in performance, with reductions reaching up to -11.86 EM score on the SQuAD dataset.
This finding highlights the inherent limitations of these methods, especially in scenarios with high knowledge consistency, where their application is particularly challenging. 

The Regular shows the least efficacy in handling Conf. data compared to Non-Conf. data\footnote{This observation does not always apply to the OPT model. This limitation is attributed to the inherent scarcity of parametric knowledge within the model. (See Appendix~\ref{appendix:distribution})}, falling by nearly 50\% on LLaMA2 model.
This observation aligns with previous research, indicating that larger models are more prone to disregard context when it conflicts with the model's parametric knowledge. 
Moreover, SC adopts the voting strategy from multiple generations. It naturally has better results on Non-Conf., but could not deal with the conflicts in Conf. By contrast, the proposed COIECD comprehensively considers the conflicts and non-conflicts between the given contexts and LLMs. As a results, it obtains the best performance on Total. And it also has better results than CD and CAD whatever on Non-Conf., and Conf. in most datasets.

In summary, whether it's SC, CD, or CAD, each is made for either Conf. or Non-Conf. scenarios, achieving comparatively better outcomes in one scenario while inevitably performing poorly in the other. 
In contrast, our adaptive decoding method considers both scenarios, achieving a trade-off that works well in all datasets. 

\subsection{Performance with Different Conflicting Data Proportions}
We conduct further experiments aiming to understand how the presence of conflicts within data affects the performance of these methods, measured in terms of EM score. We establish two experimental scenarios: a real-world conflicting scenario composed of samples from conflicting and non-conflicting data in the NQ dataset, and a synthetic conflicting scenario sampled by the same non-conflicting data and the synthetic Counterfacts constructed on the NQ.
As shown in Figure~\ref{fig:line1} and~\ref{fig:line2}, 
we visualize the correlation among the proportion of conflicting data and the performance of different methods in two scenarios.

\begin{figure} [t]
  \begin{minipage}[t]{0.42\linewidth}
  \begin{flushleft}
    \includegraphics[scale=0.2]{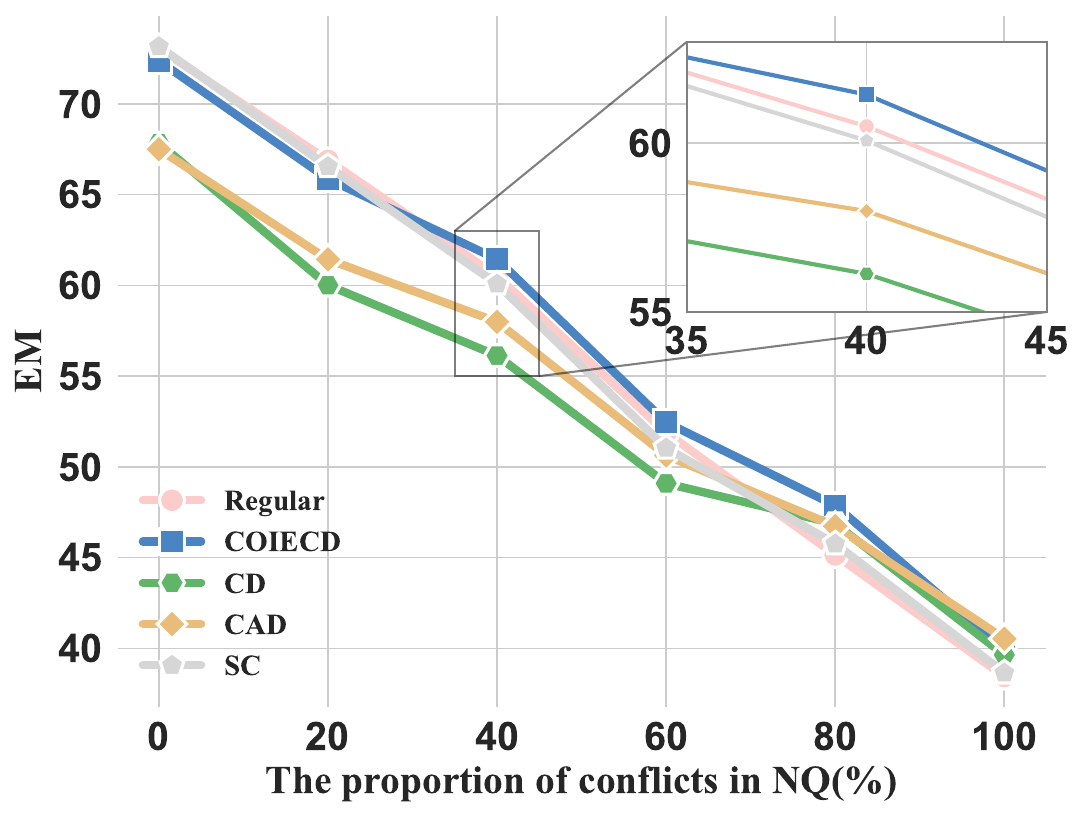}
    \captionsetup{font=small}
    \caption{Realistic conflicts with Conf. data on NQ}
    \label{fig:line1}
\end{flushleft}
  \end{minipage}
  \hspace{4mm}
  \begin{minipage}[t]{0.42\linewidth}
   \begin{flushright}
    \includegraphics[scale=0.25]{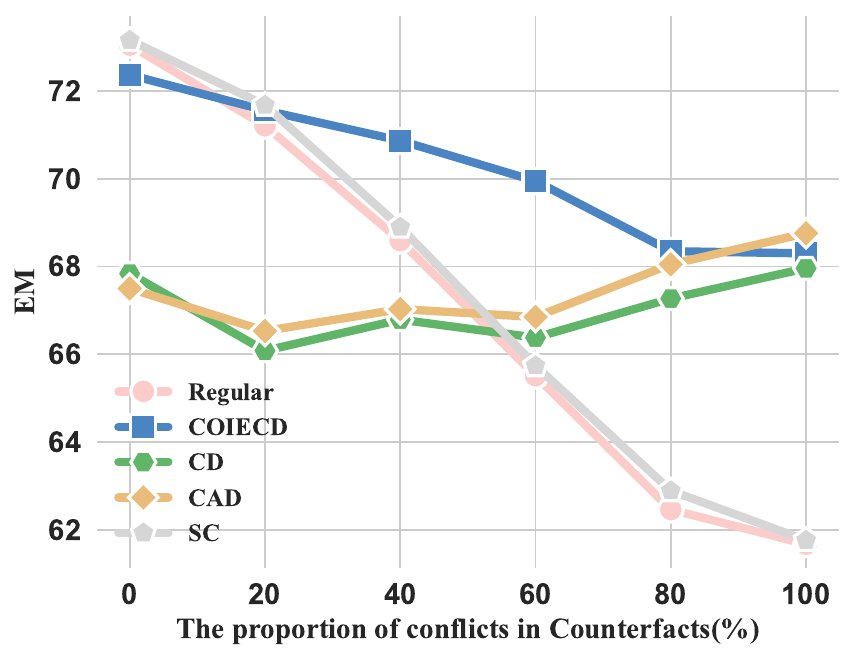}
    \captionsetup{font=small}
    \caption{Synthetic conflicts with Counterfacts}
    \label{fig:line2}
\end{flushright}
  \end{minipage}
\end{figure}
\paragraph{Performance degradation across conflict proportions.}
Both figures reveal a universal trend of performance deterioration for Regular as the conflict proportion escalates. 
The Regular and SC show the highest initial EM score at 0\% conflicts but also demonstrate the most significant decline as conflicts increase. It suggests that those methods heavily rely on knowledge consistency with parameters and contexts.
CAD exhibits the lowest performance across all levels of conflict, indicating that it may be specifically designed for datasets with maximal conflicts.
The performance of COIECD declines at the slowest rate, suggesting it has the capability that mitigate the impact of conflicting data. Overall, COIECD appears to be more robust to conflicts compared to others. 

\paragraph{Gap between realistic and synthetic scenarios.}
Upon closer inspection of Figure~\ref{fig:line1} and \ref{fig:line2}, we find that the performances of CAD and CD exhibit substantial variation with the increase of conflicts. 
In the synthetic scenario, they fall below that of Regular by a large margin when no conflict occurs, but rise gradually with increase of the proportion of knowledge conflicts. This trend does not exist in the realistic data. 
Furthermore, the impact of conflicts on EM is more pronounced in the realistic scenario. This might be due to the nature of realistic conflicts being more challenging or nuanced compared to the synthetic ones.
In conclusion, the capability of the decoding method cannot be only verified by the performance on the single counterfactual data.
To address a more realistic scenario, the COIECD method emerges as the optimal choice.
\begin{table}[t]
  \scriptsize
  \centering \ra{1.2} \scalebox{0.75}{
  \begin{threeparttable}
  \begin{tabular}{@{}p{0.14\textwidth}ccccccc@{}}
  
  \toprule
  \textbf{DPR} & \textbf{Regular} & \textbf{SC} & \textbf{CD} & \textbf{CAD} & \textbf{COIECD} \\
  \midrule
  w/o reranker$\dagger$ & & & & & \\
  \hdashline
  \quad EM & \cellcolor[gray]{0.9}16.80 & \cellcolor{green!20}16.74 & \cellcolor{green!20}15.97 &\cellcolor{green!20}16.23 &\cellcolor{red!20}\textbf{16.84}\\
  \quad F1  & \cellcolor[gray]{0.9}\textbf{22.93} & \cellcolor{green!20}22.75  & \cellcolor{green!20}22.05 & \cellcolor{green!20}22.14 & \cellcolor{green!20}22.88  \\ 
  \midrule
  \scriptsize{w/ oracle reranker} & & & & & \\
  \hdashline
  \quad EM & \cellcolor[gray]{0.9}34.92 & \cellcolor{red!20}35.24 & \cellcolor{green!20}34.20 &\cellcolor{green!20}34.10 &\cellcolor{red!20}\textbf{35.82}\\
  \quad F1  & \cellcolor[gray]{0.9}43.35 & \cellcolor{green!20}43.23  & \cellcolor{green!20}43.49 & \cellcolor{green!20}43.27 & \cellcolor{red!20}\textbf{44.48}  \\
  \bottomrule
  \end{tabular}
   \begin{tablenotes}
        \item[$\dagger$] The accuracy of Hits@1 w/o reranker is 0.46.
 \end{tablenotes}
 \end{threeparttable}
}
  \caption{Performance evaluation with DPR on Conf. data of NQ Open. The red cell indicates superior performance than the Regular decoding, and green denotes degeneration}
  \label{tab:rag}
\end{table}
\subsection{Performance on Noisy Contexts}
In this paper, the input contexts are regarded as high quality and containing the answer following the settings in~\cite{LongprePCRD021,prompting, cad,Wangyike}. However, in real-world models like retrieval-augmented language models (RALMs), contextual knowledge can be of low quality or noisy. Therefore, we also incorporate a prominent retrieval system DPR~\cite{DBLP:conf/emnlp/KarpukhinOMLWEC20} into our research on the NQ Open dataset\footnote{We use a single document as the context input, which is top-scored passage retrieved by DPR from WikiText-103.}.
One primary objectives of the RALM method is to ascertain whether a given question necessitates retrieval augmentation~\cite{POPQA,DBLP:conf/emnlp/JiangXGSLDYCN23}, which drives far from the our focus. Therefore, we conduct experiments amidst the conflicting data, where the model lacks the requisite knowledge to formulate an accurate response and necessitates retrieval augmentation.

In Table~\ref{tab:rag}, the 'w/o reranker' means the presence of noise, whereas the 'oracle reranker' has the capability to filter out all the noise. It is evident that the noise in the context significantly impacts both the CAD and CD, resulting in performances considerably lower than Regular. SC still displays the performance comparable to Regular.
In contrast, COIECD maintains a marginal superiority over Regular, a distinction that becomes more pronounced when the retriever is coupled with an oracle reranker.
Moreover, in real-world scenarios characterized by potentially noisy contexts, we posit the challenge of mitigating noise presents a unique research concern, particularly focusing on other components of RALMs, such as retrievers and rerankers. 
Enhanced reranking of external context is observed to correlate with improved performance in the COIECD. Notably, our approach still demonstrates robustness even in the absence of reranker.

\begin{table}[t]
  \small
  \centering \ra{1.2} \scalebox{0.78}{
  \begin{tabular}{@{}p{0.14\textwidth}ccccccc@{}}
  \toprule
  & \multicolumn{2}{c}{ \makecell{\textbf{Nucleus}\\($p=0.9$)} } & \multicolumn{2}{c}{ \makecell{\textbf{Top-k}\\($k=50$) }}& \multicolumn{2}{c}{ \makecell{\textbf{Typical}\\($\tau=0.9$)}}\\
  \textbf{Decoding} & EM & F1& EM & F1& EM & F1& \\
  \midrule
  COIECD  & 46.19 & 62.13  & 46.16 & 61.87 & 46.74 & 62.03  \\ 
  \quad-w/o upper  & 46.08 & 61.87 & 45.06 & 61.13 & 46.24 & 61.91  \\
  \quad-w/o lower  & 44.11 & 60.22 & 41.93 & 59.05 & 44.77 & 60.49\\
  \hdashline
  Regular  & 43.80 & 59.75 & 41.64 & 58.27 & 44.32 & 60.24 \\
  \bottomrule
  \end{tabular}}
  \caption{Performance evaluation for the ablation studies of single-side constraint on NQ dataset.}
  \label{tab:ablation}
\end{table}
\subsection{Analyses on Contextual Information-Entropy Constraint} \label{exp.3}

In this section, we delve into the criticality of the contextual information-entropy constraints within the COIECD model, specifically focusing on the impacts of the lower and upper bounds in various stochastic sampling decoding contexts. Table \ref{tab:ablation} presents the experimental results on the NQ dataset with LLaMA2-13B model.

We observe that the exclusion of the lower bound leads to a discernible decrement in both EM and F1 scores across diverse decoding strategies. It demonstrates the pivotal role of the lower bound in improving the faithfulness to the conflicting contexts. 
Although the upper bound is crucial for limiting the inclusion of low-probability, potentially irrelevant tokens, the lower bound's contribution to steering the model distribution towards more context-faithful tokens is more pronounced. Furthermore, a detailed case study is presented in Appendix~\ref{appendix:case}.

Notably, COIECD consistently surpasses the performance of Regular. This superiority is sustained even in scenarios where one of the bounds is omitted, highlighting the overall effectiveness and robustness of the COIECD. 

\begin{figure} [t]
  \begin{minipage}[t]{0.42\linewidth}
  \begin{flushleft}
    \includegraphics[scale=0.23]{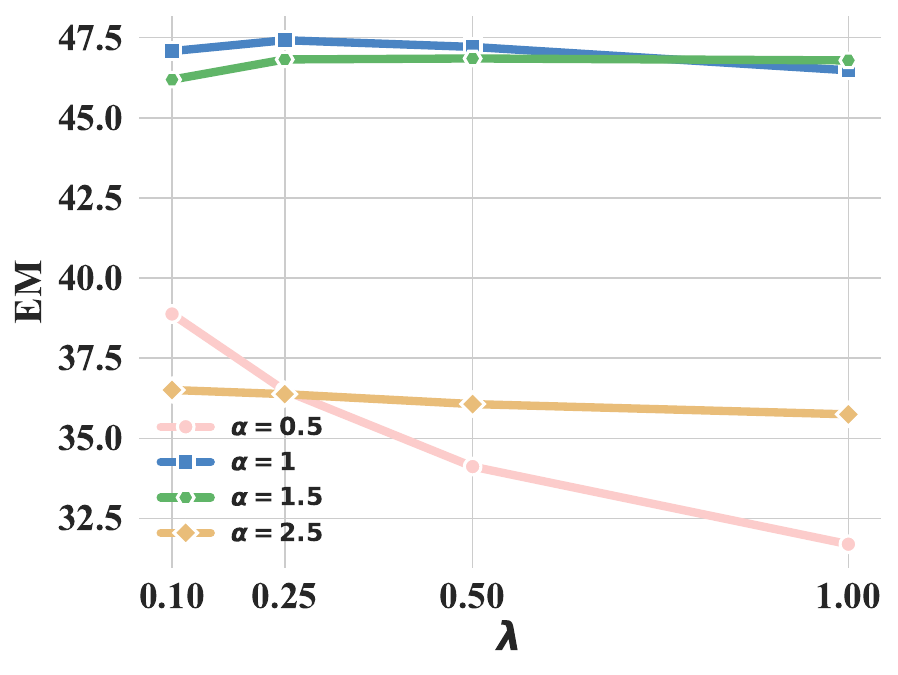}
    \captionsetup{font=small}
    \caption{EM score on LLaMA2-\textbf{13B} Model}
    \label{fig:13B}
\end{flushleft}
  \end{minipage}
  \hspace{4mm}
  \begin{minipage}[t]{0.42\linewidth}
   \begin{flushright}
    \includegraphics[scale=0.23]{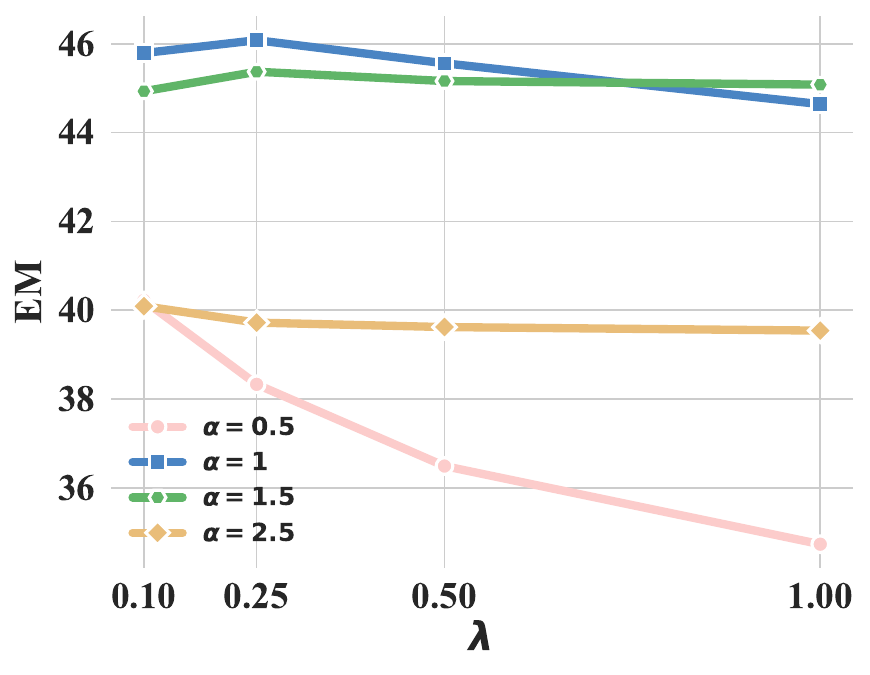}
    \captionsetup{font=small}
    \caption{EM score on LLaMA2-\textbf{7B} Model}
    \label{fig:7B}
\end{flushright}
  \end{minipage}
\end{figure}


\subsection{Discussion on Hyperparameters} \label{exp:hyper}
As illustrated in Figure~\ref{fig:13B} and \ref{fig:7B}, we conduct experiments with various values of $\lambda$ and $\alpha$ on NQ dataset. We find $\lambda=0.25$ and $\alpha=1$ consistently provide robust improvements over Regular decoding. Therefore, we adopt this hyperparameter configuration across all experiments.

Furthermore, we evaluate the model performance under the setting of $\alpha=0$ as simply providing not providing the context during decoding in Appendix~\ref{appendix:hyper}, which highlight the significance of adding $g(y_t)$. The detailed results are evaluated with EM and F1 metrics in Table~\ref{table:em-13b} - \ref{table:f1-7b}.

\section{Conclusion}
The COIECD method is introduced to discern and resolve knowledge conflicts effectively. This method is evaluated on context-relevant QA tasks using both realistic and synthetic datasets. The findings indicate that COIECD maintains consistently high performance, irrespective of the presence or absence of knowledge conflicts within the data.

\section{Limitations}
\begin{itemize}
    \item We only evaluate our decoding method on the tasks of QA. It would be interesting to apply our method to other context-intensive NLP tasks such as summarization~\cite{DBLP:conf/acl/MaynezNBM20, DBLP:conf/naacl/PagnoniBT21}.
    \item Similar to the limitations of CD and CAD, our method also requires twice the computational resources due to the necessity of performing two decoding operations, thus resulting in a cost equivalent to double that of Regular decoding.
    \item Given the shorter length of answers in QA tasks, our approach omits the entropy smoothing calculation within the constraint during the decoding process. This step is generally incorporated in open-ended text generation tasks, aligning with the stable entropy theory described by \citet{stableentropy}. Although this adaptation is practical for QA, we recognize it as a limitation and propose it as an area for future research.
\end{itemize}

\section*{Acknowledgments}
This work was supported by the National Key R\&D Program of China (No.2022ZD0160503) and Yunnan Provincial Major Science and Technology Special Plan Projects (No. 202202AD080004).
\bibliography{custom}

\clearpage
\appendix

\section{Analyses of the Performances of Existing Decoding Methods} \label{appendix:analysis}
We compare the performance of three baseline methods (introduced in Section~\ref{baselines}) on LLaMA2-13B model: Regular takes the context and question as input with greedy decoding, and the other two methods are specialized in conflict-solving decoding strategies.
We experiment on three subsets of NQ~\cite{kwiatkowski2019natural}: data without conflict($\sim$1K), data with all conflicts ($\sim$3K), and random sampled data with half-ratio conflicts. The details of these datasets is introduced in Appendix.~\ref{appendix:data} and the detailed experimental results are illustrated in the NQ dataset of Table~\ref{tbl:main} and Figure~\ref{fig:line1}. We use Exact Match (EM) as the major evaluation metric in the comparison.

In the table presented below, it is observed that when conflicts occur 100\% of the time, both conflict-solving decoding methods address the issue more effectively than Regular. Notably, CAD exhibits a pronounced improvement, achieving a significant increase of up to 2.08 in the EM score. Nevertheless, when the ratio of conflict decreases, there is a discernible decrease in those methods' efficacy. Especially, the performance of CAD noticeably deteriorates, trailing behind the Regular by a margin of 5.55.

\section{Information-Theoretic Properties of Language Models} \label{properties}
\subsection{Locally Typical Set}
\citet{typicalsampling} posit that the language modeling can be conceptualized as a discrete stochastic process and 
build its notion on the concept of \textbf{typical set}. Informally, the typical set, derived from information theory, is the set of all samples that we would expect when sampling from the language model distribution. 
But it relies on a stationary and ergodic language process which contradicts with the non-ergodic language process. 
So they define a more restrictive notion of typical set - termed as \defn{locally typical set} - for the language process, from which each token generates in a natural and error-minimizing manner.


\begin{definition}[Locally Typical Set]\label{def:local-typical}
Let $\boldY = \{Y_t\}_{t=1}^\infty$ be a discrete stochastic process under distribution $p$.
The $(T,\varepsilon)$-\defn{locally typical set} of $\boldY$ is
the set of all sequences of length exactly $T$ such that\looseness=-1
\begin{align}
\mathcal{L}_\varepsilon^{(T)} &= \Big\{ \yy = \boldsymbol{y}_0\cdots \boldsymbol{y}_T \mid \forall  1 \leq t \leq T, \\
    & \Big|\log p(y_t \mid \yy_{<t}) + \ent(Y_t \mid \boldY_{<t} = \yy_{<t}) \Big| < \varepsilon \Big\} \nonumber
\end{align}
\end{definition}
The relationship can be formalized as the following hypothesis, which has been verified empirically using data from human language process.
\begin{hypothesis}\label{hyp:info-rate}
Samples $\yy = \boldsymbol{y}_0 \cdots \boldsymbol{y}_T $ from a human language process with distribution $p$ tend to belong to the process's locally typical set $\mathcal{L}_\varepsilon^{(T)}$ for large enough $T$ and some $\varepsilon > 0$.
In words, this means that we should expect every word in natural-sounding sentences to be close to the \emph{expected} information content under $p$, i.e., the conditional entropy given prior context.
\end{hypothesis}
The $\ent$ represents the entropy rate of $\boldY$, which is equivalent to the standard definition of (Shannon) entropy $\mathcal{H}$ for a random variable $\boldY$.
The locally typical set restricts the set of tokens to those for which each has an information context——measured by its negative log probability——close to the \textit{expected} information content given prior context, i.e., the entropy of the distribution $p(\cdot \mid \yy_{<t})$.

\subsection{Stable Entropy Hypothesis}
\citet{stableentropy} postulate that natural language generations usually lie in a narrow and nearly flat entropy band. 
In the empirical analyses, they observe that, the mean entropy of a language model remains stable over the length of the generation, 
which is defined as the \textbf{stable entropy baseline}\footnote{Here we drop the smoothing step for brevity.} in Eq.\ref{eq:basline}.  
Under the context distribution at time $t$, an input $\boldsymbol{x}$ and vocabulary $\mathcal{V}$, $y_t \in \mathcal{V}$:
\begin{equation}\label{eq:basline}
  \mu_{\mathcal{H}}(t;\mathcal{V}) = \mathbb{E}_{y_t \in \mathcal{V}} \bigl[  {\mathcal{H}}(y_t\mid \boldsymbol{x}, \boldsymbol{y}_{<t}) \bigl].
\end{equation}

Then a \textbf{stable entropy zone} is defined as the zone around the stable entropy baseline that covers a major fraction of entropy of the model under the target distribution. 
They define it by standard deviation ($\sigma_{\mathcal{H}}(t;\mathcal{V})$) around the stable entropy baseline as the stable entropy zone and posit the following hypothesis:
\begin{hypothesis} \label{hyp:stable}
Decoding algorithms whose generation’s smoothed entropy stays mostly enclosed within the stable entropy zone will produce higher quality, coherent, less repetitive, and more "human-like" text.
\end{hypothesis}

\section{Detailed Proofs of Propositions} \label{derivations}
\begin{assumption} \label{assumption}
If a task-specific context $\boldsymbol{c}$ is contained by parametric knowledge (denoted as $\mathcal{K}$) without triggering any conflicts in model $p$, then it also can be the natural generation of model. 
\begin{equation*}
    \text{if } \boldsymbol{c} \in \mathcal{K} \text{, then } \boldsymbol{c} \in \bigcup \boldsymbol{y}\sim p(\cdot \mid \boldsymbol{x})
\end{equation*}
\end{assumption}
Here, $\bigcup \boldsymbol{y}$ indicates the sampling set of all natural generations by the model conditioning by the question $\boldsymbol{x}$. 
Then we define the entropy of the model following~\citet{stableentropy} as
\begin{equation}
    \mathcal{H}(\boldsymbol{y}_t|\boldsymbol{x},\boldsymbol{y}_{<t}) = \underset{y_t \sim p(\cdot| \boldsymbol{y}_{<t})}{\mathbb{E}}  -\log p(y_t|\boldsymbol{x},\boldsymbol{y}_{<t}) 
\end{equation}
For brevity, we use $\mathcal{H}_{1}(\boldsymbol{y_t})$ to represent the entropy of conditional distribution over question $\boldsymbol{x}$ and generation $\boldsymbol{y}_{<t}$, and $\mathcal{H}_{2}(\boldsymbol{y_t})$ denotes the entropy conditioning by $\boldsymbol{x}$, $\boldsymbol{y}_{<t}$, and assumed generation $\boldsymbol{c}$.
\begin{align}
    \mathcal{H}_{1}(\boldsymbol{y_t}) &= \mathcal{H}(\boldsymbol{y}_t|\boldsymbol{x},\boldsymbol{y}_{<t})\\
    \mathcal{H}_{2}(\boldsymbol{y_t}) &= \mathcal{H}(\boldsymbol{y}_t|\boldsymbol{x},\boldsymbol{c},\boldsymbol{y}_{<t}) 
\end{align}
where $\mathcal{H}_{2}(\boldsymbol{y_t})$ denotes the entropy conditioning by previously generated tokens $\boldsymbol{c}$ and $\boldsymbol{y}_{<t}$, and $\mathcal{H}_{1}(\boldsymbol{y_t})$ represents the entropy of conditional distribution over generation $\boldsymbol{y}_{<t}$. 
\begin{proposition}[Bound on Entropy Shift]
The entropy shift denoted as $\mathcal{H}_{2}(\boldsymbol{y_t}) - \mathcal{H}_{1}(\boldsymbol{y_t}) $ is bounded within the width of the \textbf{stable entropy zone}.
\end{proposition}

\begin{proof}
Note that context $c$ is the natural generation of a language model in the setting, both the entropy $\mathcal{H}_{1}(\boldsymbol{y_t})$ and $\mathcal{H}_{2}(\boldsymbol{y_t})$ should fall into a stable entropy zone around the mean entropy $\mu_{\mathcal{H}}$. Let $\beta$ be the threshold of a certain standard deviation around the mean entropy. According to Eq.~\ref{eq:basline}, it can be deduced that
\begin{equation} \label{stable zone}
\big | \mathcal{H}_{1}(\boldsymbol{y_t}) - \mu_{\mathcal{H}_{1}} \big | < \frac{\beta}{2},\
\big | \mathcal{H}_{2}(\boldsymbol{y_t}) - \mu_{\mathcal{H}_{2}} \big | < \frac{\beta}{2}
\end{equation}
\textbf{Stable entropy baseline} demonstrates that mean entropy of a model under the target context distribution remains stable. Since the length of context is limited, the mean entropy $\mu_{\mathcal{H}_{1}}$ and $\mu_{\mathcal{H}_{2}}$ can be equated if smoothed, denoted as $\mu_{\mathcal{H}}$. Considering inequalities in Eq.(\ref{stable zone}) jointly, we can obtain the bound on the entropy shift using the triangle inequality: 
\begin{align} \label{stable ineq}
&\ \quad \big | \mathcal{H}_{2}(\boldsymbol{y_t}) - \mathcal{H}_{1}(\boldsymbol{y_t}) \big | \nonumber\\
&= \big | \big (\mathcal{H}_{2}(\boldsymbol{y_t}) - \mu_{\mathcal{H}}\big ) - \big(\mathcal{H}_{1}(\boldsymbol{y_t}) - \mu_{\mathcal{H}}\big) \big |\nonumber \\ 
& < \big | \mathcal{H}_{2}(\boldsymbol{y_t}) - \mu_{\mathcal{H}} \big | + 
\big | \mathcal{H}_{1}(\boldsymbol{y_t}) - \mu_{\mathcal{H}} \big | <  \beta
\end{align}
\end{proof}

\begin{proposition}[Bound on information-entropy shift]
As the information content of a random variable is quantified as its negative log-probability.
Let the information content $I(y_t)=-\log p(y_t\mid \boldsymbol{x},\boldsymbol{c},\boldsymbol{y}_{<t})$, we denote the \textbf{information-entropy shift} as: $I(y_t) - \mathcal{H}_{1}(\boldsymbol{y_t}) $. The following bound holds for a constant:
\begin{equation} 
    \big | I(y_t) - \mathcal{H}_{1}(\boldsymbol{y_t}) \big | < \gamma 
\end{equation}
where $\gamma > 0$.
\end{proposition}

\begin{proof}
\textbf{Locally typicality} demonstrates that the information content of $y$ should is quite close to a specific value of the entropy under model distribution $p$. It means that there exists a sufficiently small constant $\epsilon > 0$:
\begin{equation}\label{typical ineq}
    \big | I(y_t)- \mathcal{H}_{2}(\boldsymbol{y_t}) \big | < \epsilon
\end{equation}
which bounds the information of $y$ into a coherent and contextual generation. 
Applying triangle inequality on Eq.(\ref{stable ineq}) and Eq.(\ref{typical ineq}), the following inequality holds for a constant:
\begin{align}
&\ \quad \big | I(y_t) - \mathcal{H}_{1}(\boldsymbol{y_t}) \big | \nonumber\\
&= \big | \big (I(y_t) - \mathcal{H}_{2}(\boldsymbol{y_t})\big ) + \big (\mathcal{H}_{2}(\boldsymbol{y_t}) - \mathcal{H}_{1}(\boldsymbol{y_t}) \big )\big |\nonumber\\
&< \big | I(y_t) - \mathcal{H}_{2}(\boldsymbol{y_t}) \big | + \big | \mathcal{H}_{2}(\boldsymbol{y_t}) - \mathcal{H}_{1}(\boldsymbol{y_t}) \big |\nonumber\\
&< \beta + \epsilon = \gamma
\end{align}
\end{proof}

\section{Empirical Study of Assumption C.1} \label{appendix:assumption}
In this section, we show that the distribution entropy of non-conflicting context remains more stable than the non-conflicting one. Then the assumption~\ref{assumption} can be proved with the \textbf{stable entropy hypothesis}~\ref{hyp:stable}.

To demonstrate our assumption, we follow a similar setup as ~\citet{stableentropy} in a text completion setup. We use the LLaMA2-13B model and NQ data. We sample 500 pieces of data from Conf. and Non-Conf. sub-datasets respectively, then compute the mean smoothed entropy at each step and calculate the standard deviation (std) for each generation. 
\begin{figure}[htbp]
\begin{center}
\includegraphics[scale=0.5]{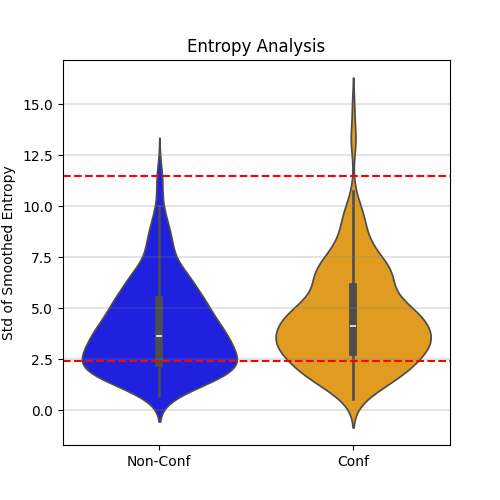} 
\caption{Conf. or Non-Conf. distributions of the 'Std of Smoothed Entropy' for NQ dataset.}
\label{fig:std}
\end{center}
\end{figure}
Figure~\ref{fig:std} visualizes the std of smoothed entropy for conflicting and non-conflicting generation. The vertical axis, labeled 'Std of Smoothed Entropy', represents the std of each step's entropy in the generation. The horizontal axis represents the NQ samples from Conf. and Non-Conf. data. From the violin plot, it can be observed that the entropy distribution for Conf. data exhibit a bimodal nature, suggesting that quite a few samples are characterized by large variances.  Furthermore, the box line of Conf. is higher than the one of Non-Conf., which demonstrates that the Non-Conf. is more likely to be a natural generation of the model due to its more stable entropy levels.
This is deduced by the \textbf{stable entropy hypothesis}, which posits that "\textit{generation’s smoothed entropy stays mostly enclosed within the stable entropy zone will produce higher quality, coherent, less repetitive, and more 'human-like' text.}".

\section{Dataset Details} \label{appendix:data}
We use three realistic QA datasets (NaturalQuestions~\cite{kwiatkowski2019natural}, SQuAD 1.1~\cite{squad}), and StrategyQA~\cite{geva2021did}) and one conflicting QA dataset (Counterfacts dataset~\cite{LongprePCRD021}) for evaluating our method. 

NaturalQuestions consists of real-world information-seeking queries issued to the Google search engine and their corresponding long answers (gold evidence passage) and short answers (one or more entities). In our study, we employ the long answers as the input context and short answers as the ground truth, and conduct evaluations on the dev set.

The SQuAD 1.1 is a common QA benchmark. It includes questions posed by human annotators on a given Wikipedia paragraph, where the answer to each question is a segment of text (or span) from the paragraph. In our experiments, we conduct experiments on the dev for evaluation. 

StrategyQA is a fact reasoning benchmark that necessitates the implicit question decomposition into reasoning steps. Built around Wikipedia terms, these questions are accompanied by multiple evidence paragraphs.  The model is expected to provide a True or False answer. We concatenate question-relevant evidences to form the input context. We adopt the training set for evaluation, considering the volume of data.

Counterfacts is based on the NaturalQuestions~\cite{kwiatkowski2019natural} dataset. To generate conflicting contextual knowledge, \citet{LongprePCRD021} first identify questions with named entity answers, find the supportive document for each question and then replace the gold answer entity in the document with a random entity.

\subsection{Posteriori judgement} \label{appendix:prior}
We delineates the process of identifying instances of knowledge conflicts. The evaluation of these conflicts is based on the accuracy\footnote{Given the excessively rigid nature of EM for evaluation, an F1 score of 0.5 has been employed as a proxy for preliminary categorization.} of the model's responses when context is not provided. The scenarios are divided into two categories:
\begin{itemize}
\item \textbf{Non-Conflicting (Non-Conf.)}: This category pertains to situations where the model is capable of accurately responding to a question without the need for its corresponding context. Such instances suggest that the model has internalized the context, thereby indicating a consistency between its parametric knowledge and the external contextual knowledge.
\item \textbf{Conflicting (Conf.)}: When the model fails to provide the true answer without the aid of context, indicating a conflict between its inherent parametric knowledge and the external contextual knowledge. Following~\citet{DBLP:conf/emnlp/WangLSL23}, incorrect responses reflects the model does not possess the knowledge equipped by the external context, which has a discrepancy with the model's parametric knowledge. 
\end{itemize}
In this setting, the NQ, SQuAD and StrategyQA datasets can serve as suitable approximations of realistic scenarios where conflicts may not necessarily occur. Additionally, the synthetic dataset named Counterfacts, which is composed exclusively of conflicting data (Conf. data), serves as a unique case. This is because it contains randomly replaced answers that are not inherently known to the model, distinguishing it from the aforementioned datasets.

\subsection{Data Statistic} \label{appendix:distribution}

\begin{table}[htbp]
\tiny
\centering
\resizebox{0.45\textwidth}{!}{
\begin{tabular}{ccclllll}
\toprule
& &  \multicolumn{2}{c}{\textbf{LLaMA2}} & \multicolumn{2}{c}{\textbf{OPT}} & \multicolumn{2}{c}{\textbf{FLAN-T5}}  \\
\cmidrule(lr){3-4}
\cmidrule(lr){5-6}
\cmidrule(lr){7-8}
\multicolumn{2}{c}{\textbf{Datasets}}  & 7B & 13B & 6.7B & 13B & 3B & 11B \\
\midrule
\multirow{3}{*}{ NQ ($\sim$4K) } & Total(\%)   & 100 & 100 & 100 & 100 & 100 & 100 \\
& Conf.(\%)   & 81.91 & 76.79 & 99.34 & 97.21 & 88.07 & 85.80 \\
& Non-Conf.(\%)   & 18.09 & 23.21 & 0.64 & 2.79 & 11.93 & 14.20 \\
\midrule
\multirow{3}{*}{ SQuAD ($\sim$6K) } & Total(\%)   & 100 & 100 & 100 & 100 & 100 & 100 \\
& Conf.(\%)   & 84.18 & 82.06 & 97.48 & 95.41 & 92.30 & 90.55 \\
& Non-Conf.(\%)   & 15.82 & 17.94 & 2.56 & 4.59 & 7.70 & 9.45 \\
\midrule
\multirow{3}{*}{ StrategyQA ($\sim$2K) } & Total(\%)   & 100 & 100 & 100 & 100 & 100 & 100 \\
& Conf.(\%)  & 40.31 & 39.43 &  94.98 & 88.91 & 36.11 & 33.23 \\
& Non-Conf.(\%)  & 59.69 & 60.57 & 5.02 & 11.09 & 63.89 & 66.77 \\
\midrule
\midrule
 Counterfacts ($\sim$6K)  & Conf.(\%)  & 100 & 100 & 100 & 100 & 100 & 100 \\
\bottomrule
\end{tabular}}
\caption{The data distributions of the datasets}
\label{tbl:distribution}
\end{table}
As illustrated in Table~\ref{tbl:distribution}, a discernible trend emerges wherein an escalation in the model's parameters is accompanied by a corresponding increase in the percentage of non-conflicting data, signifying a greater degree of internalized knowledge within larger models. Notably, among this cohort of models, the OPT series models exhibit the lowest parametric knowledge, yet they demonstrate substantial enhancements across most datasets when the COIECD method is applied. It is also noteworthy to observe that even in the case of the popular LLaMA2 models, the proportion of non-conflicting data does not surpass 25\% in the NQ and SQuAD datasets. This observation necessitate the further research for the inherent parametric knowledge enhancement of the model.

\section{Baseline Methods} \label{appendix:methods}
\paragraph{Contrastive Decoding (CD)} In our experiments, we employ the distribution $g(y_t)$ with a certain threshold as a baseline decoding method, referred to as the CD~\cite{cd} method. We modify the original object of CD (computes the distribution discrepancy between an small amateur model and an expert larger model) to simulate the form of $g(y_t)$.  
\begin{align*}
CD_{\text{original}} &= \log p_{\text{EXP}}(y_t|x,y<t) - \\
    & \quad \log p_{\text{AMA}}(y_t|x,y<t) \\ 
CD_{\text{modify}} &= \log p(y_t|x,y<t)-p(y_t|y<t) \\
&= \log g(y_t) 
\end{align*}
The threshold is same as in the original CD method:
\begin{align*}
&\mathcal{V}_{\text{head}}(y_{<t}) = \\
&\left\{ y_t \in \mathcal{V} : p(y_t | y_{<t}) \geq 0.1 \cdot \max_{y} p(y | y_{<t}) \right\} \nonumber
\end{align*}
Here, we represent the input context as $x$. CD adopts the object of difference between the output likelihood when inputs are presented with and without input context. It enhances the influence of the context for high-probability words within a crude threshold. 
Therefore, it cannot obtain consistent improvement in performance, particularly with non-conflicting data. 

And the Section~\ref{sec:constraint} aims to explore a delicate constraint for output distribution to find out whether the context is in conflict. Then we propose a contextual information-entropy constraint on fine-grained token level based on the perspective of information theory.
$$
C(y_{<t}) = \{ y \in \mathcal{V} : l_{p_s} \leq p_s(y_t) \leq u_{p_s} \}\quad (7) \nonumber
$$

\paragraph{Context-Aware Decoding (CAD)}
In CAD~\cite{cad} method, the output probability is a product-of-experts of the original output probability and PMI weighted by $\alpha$ as follow:
\begin{align*}
    &y_t \sim \operatorname{softmax}[(1+\alpha) \operatorname{logit}_\theta(y_t \mid \boldsymbol{c}, \boldsymbol{x}, \boldsymbol{y}_{<t})\\
    &\qquad \qquad - \alpha \operatorname{logit}_\theta(y_t \mid \boldsymbol{x}, \boldsymbol{y}_{<t})]
\end{align*} 
Since they set $\alpha$ = 1 for all models evaluated on the knowledge conflict datasets, this method can be regarded as an unconstrained ($\lambda=1 \; in\; C(y_{<t})$) decoding method when $\alpha$ is set to 1.
If so, CAD can be considered as a specific case of our approach.

Furthermore, CAD, as evidenced in their experimental evaluation, necessitates the different hyperparameter values (the adjustment level of CAD is 0.5 and 1) for realistic datasets and counterfacts. The absence of such specific adjustments results in a substantial decline in performance. This aspect of our findings underscores the superiority and robustness of our method.
\section{Maximization v.s. Sampling Strategies} \label{appendix:strategy}
Recall that prior experiments are conducted based on greedy strategy that maximizes the distribution probability, except for SC with a fixed sampling strategy. We explore other strategies like sampling alternatives based on the same baselines. Table \ref{tab:sampling} represents the results on maximization-based strategies: greedy decoding, and stochastic sampling: nucleus \cite{holtzman2019curious}, top-k \cite{LewisDF18}, typical \cite{typicalsampling} 
on the NQ dataset of LLaMA2-13B.

We observe that COIECD consistently produces the higher EM and F1 score the than Regular irrespective of the choice of decoding strategy. In contrast, both the CD and CAD exhibit a lack of stability in performance among diverse decoding strategies. Additionally, the result points to the significant value of beam search, particularly in relation to CD, in boosting performance. It can be attributed to the increasing search width, a feature of beam search which effectively eliminates disturbing tokens brought by contrastive object.
\renewcommand{\arraystretch}{1.2}
\begin{table}[htbp]
  \small
  \ra{1.1} \scalebox{0.65}{
  \begin{tabular}{@{}p{0.14\textwidth}cccccccc@{}}\toprule
  \multirow{2}{*}{\makecell[l]{\textbf{Decoding} \\ \textbf{Methods}}}& \multicolumn{2}{c}{ \textbf{Regular} } & \multicolumn{2}{c}{ \textbf{CD}} & \multicolumn{2}{c}{ \textbf{CAD}} & \multicolumn{2}{c}{ \textbf{\ours}}\\
    & EM & F1& EM & F1& EM & F1& EM & F1\\
  \midrule
  \text{Greedy} & \cellcolor[gray]{0.9}46.48 & \cellcolor[gray]{0.9}61.51 & \cellcolor{green!20}46.19 & \cellcolor{red!20}61.97 & \cellcolor{red!20}46.79 & \cellcolor{red!20}62.29 & \cellcolor{red!20}\textbf{47.42} & \cellcolor{red!20}\textbf{62.89} \\

  \cline{1-9}
  
  \makecell[l]{Nucleus \\ \quad ($p=0.9$)}   & \makecell[l]{ \\ \cellcolor[gray]{0.9}43.80} & \makecell[l]{ \\ \cellcolor[gray]{0.9}59.75}  & \makecell[l]{ \\ \cellcolor{red!20}46.14}& \makecell[l]{ \\ \cellcolor{red!20}61.73} & \makecell[l]{ \\ \cellcolor{red!20}44.37} & \makecell[l]{ \\ \cellcolor{red!20}60.50} & \makecell[l]{ \\ \cellcolor{red!20}\textbf{46.19}} & \makecell[l]{ \\ \cellcolor{red!20}\textbf{62.13}}\\

  \quad ($p=0.95$)  & \cellcolor[gray]{0.9}43.82 & \cellcolor[gray]{0.9}60.05  & \cellcolor{red!20}45.77 & \cellcolor{red!20}62.03 & \cellcolor{green!20}43.17 & \cellcolor{green!20}59.45 & \cellcolor{red!20}\textbf{46.53} & \cellcolor{red!20}\textbf{62.80}\\
  \hdashline
  \makecell[l]{Top-k\\ \quad $(k=30)$}        & \makecell[l]{ \\ \cellcolor[gray]{0.9}42.46} & \makecell[l]{ \\ \cellcolor[gray]{0.9}58.64}  & \makecell[l]{ \\ \cellcolor{red!20}46.03} & \makecell[l]{ \\ \cellcolor{red!20}61.86} & \makecell[l]{ \\ \cellcolor{green!20}41.88} & \makecell[l]{ \\ \cellcolor{green!20}58.37} & \makecell[l]{ \\ \cellcolor{red!20}\textbf{46.98}} & \makecell[l]{ \\ \cellcolor{red!20}\textbf{62.14}}\\
  \quad $(k=50)$  & \cellcolor[gray]{0.9}41.64 & \cellcolor[gray]{0.9}58.27  & \cellcolor{red!20}45.37 & \cellcolor{red!20}61.42 & \cellcolor{red!20}41.82 & \cellcolor{red!20}58.54 & \cellcolor{red!20}\textbf{46.16} & \cellcolor{red!20}\textbf{61.87}\\
  \hdashline
  \makecell[l]{Typical\\ \quad $(\tau=0.2)$}        & \makecell[l]{ \\ \cellcolor[gray]{0.9}45.08} & \makecell[l]{ \\ \cellcolor[gray]{0.9}61.03}  & \makecell[l]{ \\ \cellcolor{red!20}46.06} & \makecell[l]{ \\ \cellcolor{red!20}61.93} & \makecell[l]{ \\ \cellcolor{red!20}45.14} & \makecell[l]{ \\ \cellcolor{green!20}60.70} & \makecell[l]{ \\ \cellcolor{red!20}\textbf{47.08}} & \makecell[l]{ \\ \cellcolor{red!20}\textbf{62.75}}\\
  \quad $(\tau=0.9)$   & \cellcolor[gray]{0.9}44.32  & \cellcolor[gray]{0.9}60.24  & \cellcolor{red!20}46.37 & \cellcolor{red!20}61.97& \cellcolor{green!20}43.77& \cellcolor{green!20}60.01&\cellcolor{red!20}\textbf{46.74} & \cellcolor{red!20}\textbf{62.03}\\
  \bottomrule
  \end{tabular}}
  \caption{\textbf{Decoding on maximization-based and stochastic sampling strategies.} The red cell indicates superior performance than the Regular decoding, and green denotes degeneration.}
  \label{tab:sampling}
\end{table}
\renewcommand{\arraystretch}{1}

\section{Case Study}\label{appendix:case}
As illustrated in Figure~\ref{fig:case}, we look closer into two cases of conflicting and non-conflicting one.
\begin{figure*}[htbp] 
\begin{center}
\includegraphics[scale=0.35]{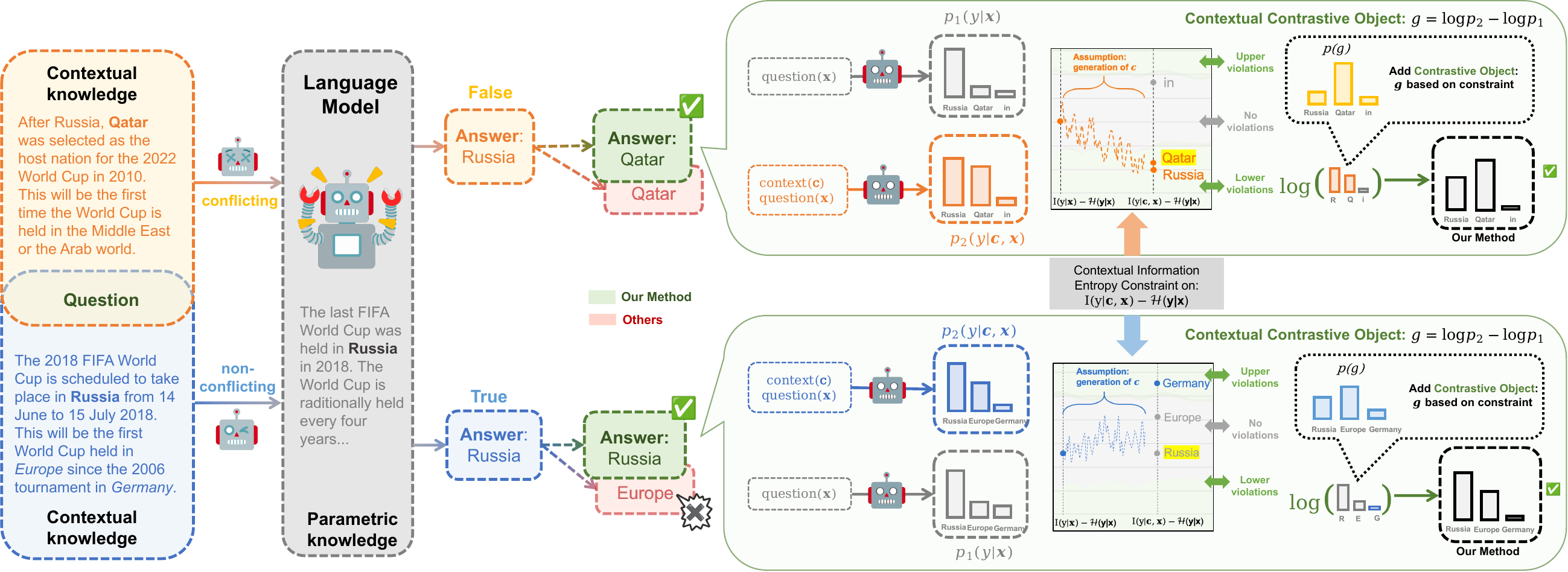} 
\caption{Left: The illustration of conflicting and non-conflicting scenarios. Existing methods adeptly handle conflicts but struggle to address non-conflicting context. In contrast, COIECD exhibits the capability to effectively handle both scenarios. Right: The detailed process of COIECD method. Utilizing a contextual information-entropy constraint, we discern the tokens that violate this constraint, which are typically triggered by conflicting contexts. For these tokens, situated in different zones, we employ distinct decoding strategies.}
\label{fig:case}
\end{center}
\end{figure*}
\paragraph{Lower-bound \& Upper-bound violation.}
The conflicting case mainly shows the function of lower bound. For token $y_t$, if $p_{\delta}(y_t) \leq l_{p_{\delta}}$, it represents a sufficiently low information content $I(y_t)$ compared to the entropy $\mathcal{H}_{1}(\boldsymbol{y_t})$. This indicates that the generation (like \textit{Russia}) may be overconfident and other informative gains (like \textit{Qatar}) may be ignored, then the conflict occurs. 
The upper bound serves to filter some disturbing low-probability distribution, which plays a role in stochastic sampling decoding. In the non-conflicting case, if $p_{\delta}(y_t) \geq u_{p_{\delta}}$, the high information context represents a lower probability, indicating that the model is less certain about current token (like \textit{Germany}). The decreased confidence might also be attributed to a potential conflict within the context.

\paragraph{No violation.}
In the non-conflicting case, 
it is observed that no tokens fall into the lower-violation zone. This can be attributed to the model's pronounced confidence in a solitary high-probability token, identified as \textit{Russia}. Such a high degree of confidence leads to the assignment of a zero value to $l_{p_{\delta}}$. The rationale behind this assignment stems from the understanding that a heightened level of confidence effectively indicates the non-existence of any conflict.

\section{Detailed Results on Hyperparameter Analysis}\label{appendix:hyper}
Here we display the detailed results about hyperparameter analysis on different sizes of LLaMA2 model with EM and F1 metrics.
\begin{table}[h!]
\small
\ra{1.1} \scalebox{1.1}{
\centering
\begin{tabular}{c|c|c|c|c}
\hline
\makecell{\textbf{EM} \\\textbf{Score}} & \textbf{$\lambda$=0.1} & \textbf{$\lambda$=0.25} & \textbf{$\lambda$=0.5} & \textbf{$\lambda$=1} \\
\hline
\textbf{$\alpha$=0}   & 19.30  & 16.82 & 15.40  & 14.25 \\
\textbf{$\alpha$=0.5} & 38.88 & 36.49 & 34.12 & 31.70  \\
\textbf{$\alpha$=1.0}   & 47.08 & \textbf{47.42 }& 47.21 & 46.48 \\
\textbf{$\alpha$=1.5} & 46.19 & 46.82 & 46.85 & 46.79 \\
\textbf{$\alpha$=2.0}   & 36.51 & 36.38 & 36.07 & 35.75 \\
\hline
\end{tabular}}
\caption{Exact Match score on LLaMA2-13B Model.}
\label{table:em-13b}
\end{table}

\begin{table}[h!]
\small
\ra{1.1} \scalebox{1.1}{
\centering
\begin{tabular}{c|c|c|c|c}
\hline
\makecell{\textbf{F1} \\\textbf{Score}} & \textbf{$\lambda$=0.1} & \textbf{$\lambda$=0.25} & \textbf{$\lambda$=0.5} & \textbf{$\lambda$=1} \\
\hline
\textbf{$\alpha$=0}   &31.67 & 27.76 & 25.40  & 23.09 \\
\textbf{$\alpha$=0.5} &56.49 & 53.97 & 51.04 & 47.30  \\
\textbf{$\alpha$=1.0}   &62.51 & \textbf{62.89} & 62.43 & 61.51 \\
\textbf{$\alpha$=1.5} &61.72 & 62.29 & 62.28 & 62.28 \\
\textbf{$\alpha$=2.0}   &54.52 & 54.34 & 53.90  & 53.56\\
\hline
\end{tabular}}
\caption{F1 score on LLaMA2-13B Model.}
\label{table:f1-13b}
\end{table}

\begin{table}[h!]
\small
\ra{1.1} \scalebox{1.1}{
\centering
\begin{tabular}{c|c|c|c|c}
\hline
\makecell{\textbf{EM} \\\textbf{Score}} & \textbf{$\lambda$=0.1} & \textbf{$\lambda$=0.25} & \textbf{$\lambda$=0.5} & \textbf{$\lambda$=1} \\
\hline
\textbf{$\alpha$=0}   & 15.19 & 13.46 & 12.28 & 11.51 \\
\textbf{$\alpha$=0.5} & 40.22 & 38.33 & 36.49 & 34.73 \\
\textbf{$\alpha$=1.0} & 45.79 & \textbf{46.08} & 45.56 & 44.64 \\
\textbf{$\alpha$=1.5} & 44.93 & 45.37 & 45.16 & 45.08 \\
\textbf{$\alpha$=2.0} & 40.09 & 39.72 & 39.62 & 39.54 \\
\hline
\end{tabular}}
\caption{Exact Match score on LLaMA2-7B Model.}
\label{table:em-7b}
\end{table}

\begin{table}[h!]
\small
\ra{1.1} \scalebox{1.1}{
\centering
\begin{tabular}{c|c|c|c|c}
\hline
\makecell{\textbf{F1} \\\textbf{Score}} & \textbf{$\lambda$=0.1} & \textbf{$\lambda$=0.25} & \textbf{$\lambda$=0.5} & \textbf{$\lambda$=1} \\
\hline
\textbf{$\alpha$=0}   &25.36 & 22.51 & 20.13 & 18.66 \\
\textbf{$\alpha$=0.5} &55.00 & 53.40  & 51.50  & 49.50  \\
\textbf{$\alpha$=1.0} &59.44 & \textbf{59.67} & 59.12 & 58.57 \\
\textbf{$\alpha$=1.5} &58.71 & 59.06 & 58.86 & 58.89 \\
\textbf{$\alpha$=2.0} &55.37 & 55.12 & 54.97 & 54.97 \\
\hline
\end{tabular}}
\caption{F1 score on LLaMA2-7B Model.}
\label{table:f1-7b}
\end{table}
The different values of alpha can measure the importance of adding $g(y_t)$ in Eq. 12 (\S~\ref{eq:g(t)}). The results highlight the significance of adding $g(y_t)$. The performance declines dramatically when $\alpha$ equals 0. It's under a decoding strategy where simply providing or not providing the context during decoding.

\section{More Results} \label{appendix:more}
We present the experimental results of GPT-3.5 and GPT-4 in Table~\ref{tbl:gpt4}, as well as the results on the other size of LLaMA2, OPT and FLAN-T5 models in Table~\ref{tbl:llama}-\ref{tbl:flant5}.

\subsection{The performances of GPT-3.5 and GPT-4}
In general, the GPT-4 model displays a modestly superior performance in comparison to the models utilized in our experiments, whereas GPT-3.5 attains a level of performance that aligns with our best results achieved by the LLaMA2-13B model.

\begin{table}[htbp]
\tiny
\centering
\resizebox{0.45\textwidth}{!}{
\begin{tabular}{ccclll}
\toprule
& &  \multicolumn{2}{c}{\textbf{GPT-3.5}} & \multicolumn{2}{c}{\textbf{GPT-4}}  \\
\cmidrule(lr){3-4}
\cmidrule(lr){5-6}
\multicolumn{2}{c}{\textbf{Datasets}}  & EM & F1 & EM & F1  \\
\midrule
\multirow{3}{*}{ NQ } & Total   & 44.45 & 61.63 & 47.36 & 65.28 \\
& Conf.  & 31.46 & 50.07 & 35.04 & 54.68 \\
& Non-Conf.  & 70.16 & 84.51 & 78.44 & 89.66 \\
\midrule
\multirow{3}{*}{ SQuAD } & Total  & 58.16 & 75.74 & 63.02 & 78.42  \\
& Conf.   & 52.39 & 71.40 & 57.58 & 75.92  \\
& Non-Conf.  & 78.07 & 90.71 & 82.63 & 93.36 \\
\midrule
\multirow{3}{*}{ StrategyQA } & Total  & 82.75 & 82.75 & 91.22 & 91.22 \\
& Conf. & 68.29 & 68.29 & 78.83 & 78.83 \\
& Non-Conf.  & 91.46 & 91.46 & 96.67 & 96.67  \\
\midrule
\midrule
 Counterfacts  & Total (Conf.)  & 61.69 & 66.40 & 64.66 & 71.11 \\
\bottomrule
\end{tabular}}

\caption{The Performances of GPT-3.5 and GPT-4}
\label{tbl:gpt4}
\end{table}

\subsection{The Performances of Models in Different Size}
Owing to the constraints of experimental resources, we confined our model within the scope of a maximum parameter capacity of 13B for the experiments.
In addition to the main results in experiment section, additional outcomes are illustrated in the Table~\ref{tbl:llama}-\ref{tbl:flant5}.

\begin{table*}[htbp]
\tiny
\centering
\resizebox{0.85\textwidth}{!}{
\begin{tabular}{cccllll}
\toprule
& & &  \multicolumn{2}{c}{\textbf{LLaMA2-7B}} & \multicolumn{2}{c}{\textbf{LLaMA2-13B}}  \\
\cmidrule(lr){4-5}
\cmidrule(lr){6-7}
\multicolumn{2}{c}{\textbf{Datasets}} & \textbf{Decoding} & EM & F1 & EM & F1 \\
\midrule
\multirow{15}{*}{ NQ } & \multirow{5}{*}{ Total } & Regular & \cellcolor[gray]{0.9}44.64 &\cellcolor[gray]{0.9}58.60 & \cellcolor[gray]{0.9}46.48 &\cellcolor[gray]{0.9}61.51  \\
& & \text{SC} & 44.72$\ddel{+0.08}$ & 58.47$\db{-0.13}$  & 46.66$\ddel{+0.18}$ & 61.76$\ddel{+0.25}$ \\
& & \text{CD} & 45.35$\ddel{+0.71}$ & 59.21$\ddel{+0.61}$ & 46.19$\db{-0.29}$ & 61.97$\ddel{+0.46}$  \\
& & \text{CAD} & 45.08$\ddel{+0.44}$ & 58.89$\ddel{+0.29}$ & 46.79$\ddel{+0.31}$ & 62.29$\ddel{+0.78}$  \\
& & \ours & \textbf{46.08}$\ddel{+1.44}$ & \textbf{59.67}$\ddel{+1.07}$ & \textbf{47.42}$\ddel{+0.94}$ & \textbf{62.89}$\ddel{+1.38}$  \\
\cmidrule(lr){2-7}
& \multirow{5}{*}{ Conf. } & Regular &\cellcolor[gray]{0.9}39.06 &\cellcolor[gray]{0.9}53.66 &\cellcolor[gray]{0.9}38.45 &\cellcolor[gray]{0.9}54.37  \\
& & \text{SC} & 38.99$\db{-0.07}$ & 53.43$\db{-0.23}$ & 38.65$\ddel{+0.20}$ & 54.64$\ddel{+0.27}$\\
& & \text{CD} & 40.69$\ddel{+1.63}$ & 55.30$\ddel{+1.64}$ & 39.64$\ddel{+1.19}$ & 56.50$\ddel{+2.13}$ \\
& & \text{CAD} & 40.82$\ddel{+1.76}$ & \textbf{55.45}$\ddel{+1.79}$  & \textbf{40.53}$\ddel{+2.08}$ & \textbf{57.15}$\ddel{+2.78}$ \\
& & \ours & \textbf{40.95}$\ddel{+1.89}$ & 55.24$\ddel{+1.58}$ & 39.88$\ddel{+1.43}$ & 56.59$\ddel{+2.22}$  \\
\cmidrule(lr){2-7}
& \multirow{5}{*}{ \makecell{Non-\\Conf.} } & Regular &\cellcolor[gray]{0.9}69.91 &\cellcolor[gray]{0.9}80.93 &\cellcolor[gray]{0.9}73.05 &\cellcolor[gray]{0.9}85.15 \\
& & \text{SC} & \textbf{70.64}$\ddel{+0.73}$ & \textbf{81.26}$\ddel{+0.33}$  & \textbf{73.16}$\ddel{+0.11}$ & \textbf{85.30}$\ddel{+0.15}$ \\
& & \text{CD} & 66.42$\db{-3.49}$ & 76.93$\db{-4.00}$ & 67.84$\db{-5.21}$ & 80.06$\db{-5.09}$  \\
& & \text{CAD} & 64.39$\db{-5.52}$ & 74.46$\db{-6.47}$  & 67.50$\db{-5.55}$ & 79.19$\db{-5.96}$ \\ 
& & \ours & 69.33$\db{-0.58}$ & 79.71$\db{-1.22}$& 72.37$\db{-0.68}$ & 83.75$\db{-1.40}$ \\
\midrule
\multirow{15}{*}{ SQuAD } & \multirow{5}{*}{ Total } & Regular &\cellcolor[gray]{0.9}54.75 &\cellcolor[gray]{0.9}68.92 &\cellcolor[gray]{0.9}54.46 &\cellcolor[gray]{0.9}68.92  \\
& & \text{SC} & 55.02$\ddel{+0.27}$ & 69.04$\ddel{+0.12}$ & 54.55$\ddel{+0.09}$ & 68.85$\db{-0.07}$  \\
& & \text{CD} & \textbf{57.56}$\ddel{+2.81}$ & \textbf{70.94}$\ddel{+2.02}$ & 53.89$\db{-0.57}$ & 68.04$\db{-0.88}$   \\
& & \text{CAD} & 56.98$\ddel{+2.23}$ & 70.12$\ddel{+1.20}$ & 56.46$\ddel{+2.00}$ & 70.52$\ddel{+1.60}$  \\ 
& & \ours & 57.32$\ddel{+2.57}$& 70.39$\ddel{+1.47}$ & \textbf{57.10}$\ddel{+2.64}$& \textbf{70.86}$\ddel{+1.94}$  \\
\cmidrule(lr){2-7}
& \multirow{5}{*}{ Conf. } & Regular & \cellcolor[gray]{0.9}50.11 & \cellcolor[gray]{0.9}65.17 & \cellcolor[gray]{0.9}48.78 & \cellcolor[gray]{0.9}64.34\\
& & \text{SC} & 50.32$\ddel{+0.21}$ & 65.25$\ddel{+0.08}$ & 48.87$\ddel{+0.09}$ & 64.24$\db{-0.10}$ \\
& & \text{CD} & \textbf{54.36}$\ddel{+4.25}$ & \textbf{68.51}$\ddel{+3.34}$ & 50.68$\ddel{+1.90}$ & 66.01$\ddel{+1.67}$  \\
& & \text{CAD} & 53.33$\ddel{+3.22}$ & 67.43$\ddel{+2.26}$ & 51.64$\ddel{+2.86}$ & \textbf{67.09}$\ddel{+2.75}$  \\ 
& & \ours & 53.41$\ddel{+3.30}$ & 67.42$\ddel{+2.25}$& \textbf{51.95}$\ddel{+3.17}$ & 66.91$\ddel{+2.57}$  \\
\cmidrule(lr){2-7}
& \multirow{5}{*}{ \makecell{Non-\\Conf.} } & Regular & \cellcolor[gray]{0.9}79.44 & \cellcolor[gray]{0.9}88.84 & \cellcolor[gray]{0.9}80.50 & \cellcolor[gray]{0.9}89.88  \\
& & \text{SC} & \textbf{80.09}$\ddel{+0.65}$ & \textbf{89.20}$\ddel{+0.36}$ & 80.57$\ddel{+0.07}$ & \textbf{89.96}$\ddel{+0.08}$  \\
& & \text{CD} & 74.57$\db{-4.87}$ & 83.85$\db{-4.99}$ & 68.64$\db{-11.86}$ & 77.35$\db{-12.53}$ \\
& & \text{CAD} & 76.41$\db{-3.03}$ & 84.44$\db{-4.40}$ & 78.53$\db{-1.97}$ & 86.19$\db{-3.69}$  \\
& & \ours & 78.14$\db{-1.30}$ & 86.21$\db{-2.63}$ & \textbf{80.69}$\ddel{+0.19}$ & 88.93$\db{-0.95}$ \\
\midrule
\multirow{15}{*}{ StrategyQA } & \multirow{5}{*}{ Total} & Regular & \cellcolor[gray]{0.9}79.69 & \cellcolor[gray]{0.9}79.69 & \cellcolor[gray]{0.9}81.09 & \cellcolor[gray]{0.9}81.09\\
& & \text{SC} & \textbf{79.34}$\db{-0.35}$ & \textbf{79.34}$\db{-0.35}$  & 81.05$\db{-0.04}$ & 81.05$\db{-0.04}$   \\
& & \text{CD} & 69.96$\db{-9.73}$ & 69.96$\db{-9.73}$ & 83.58$\ddel{+2.49}$ & 83.58$\ddel{+2.49}$ \\
& & \text{CAD} & 74.93$\db{-4.76}$ & 74.93$\db{-4.76}$ &  85.50$\ddel{+4.41}$ & 85.50$\ddel{+4.41}$\\ 
& & \ours & 78.91$\db{-0.78}$ & 78.91$\db{-0.78}$ &   \textbf{85.76}$\ddel{+4.67}$ &\textbf{85.76}$\ddel{+4.67}$\\
\cmidrule(lr){2-7}
& \multirow{5}{*}{ Conf. } & Regular & \cellcolor[gray]{0.9}61.11 & \cellcolor[gray]{0.9}61.11 & \cellcolor[gray]{0.9}57.36 & \cellcolor[gray]{0.9}57.36 \\
& & \text{SC} & 61.11$\ddel{+0.00}$ & 61.11$\ddel{+0.00}$ & 57.59$\ddel{+0.23}$ & 57.59$\ddel{+0.23}$ \\
& & \text{CD} & 59.15$\db{-1.96}$ & 59.15$\db{-1.96}$  & \textbf{81.15}$\ddel{+23.79}$ & \textbf{81.15}$\ddel{+23.79}$   \\
& & \text{CAD} & \textbf{64.57}$\ddel{+3.46}$ & \textbf{64.57}$\ddel{+3.46}$ & 77.31$\ddel{+19.95}$ & 77.31$\ddel{+19.95}$ \\
& & \ours & 63.71$\ddel{+2.60}$ & 63.71$\ddel{+2.60}$ & 80.29$\ddel{+22.93}$ & 80.29$\ddel{+22.93}$  \\
\cmidrule(lr){2-7}
& \multirow{5}{*}{ \makecell{Non-\\Conf.} } & Regular & \cellcolor[gray]{0.9}92.25 & \cellcolor[gray]{0.9}92.25 & \cellcolor[gray]{0.9}96.54 & \cellcolor[gray]{0.9}96.54 \\
& & \text{SC} & \textbf{91.66}$\db{-0.59}$ & \textbf{91.66}$\db{-0.59}$ & \textbf{96.47}$\db{-0.07}$ & \textbf{96.47}$\db{-0.07}$ \\
& & \text{CD} & 77.25$\db{-15.00}$ & 77.25$\db{-15.00}$ & 85.16$\db{-11.38}$ & 85.16$\db{-11.38}$ \\
& & \text{CAD} & 81.93$\db{-10.32}$ & 81.93$\db{-10.32}$ & 89.33$\db{-7.21}$ & 89.33$\db{-7.21}$ \\
& & \ours & 89.17$\db{-3.08}$  & 89.17$\db{-3.08}$  & 90.80$\db{-5.74}$  & 90.80$\db{-5.74}$  \\
\midrule
\midrule
\multirow{5}{*}{ Counterfacts } & \multirow{5}{*}{ \makecell{Total\\(Conf.)}} & Regular & \cellcolor[gray]{0.9}67.86 &\cellcolor[gray]{0.9}68.77 &\cellcolor[gray]{0.9}61.67 &\cellcolor[gray]{0.9}62.63 \\
& & \text{SC} & 68.30$\ddel{+0.44}$ & 69.23$\ddel{+0.46}$  &  61.76$\ddel{+0.09}$ & 62.76$\ddel{+0.13}$ \\
& & \text{CD} & 72.94$\ddel{+5.08}$& 74.29$\ddel{+5.52}$ & 67.96$\ddel{+6.29}$& 69.16$\ddel{+6.53}$ \\
& & \text{CAD} & \textbf{73.11}$\ddel{+5.25}$ & \textbf{75.99}$\ddel{+7.22}$ & \textbf{68.76}$\ddel{+7.09}$ & \textbf{71.20}$\ddel{+8.57}$   \\ 
& & \ours & 71.57$\ddel{+3.71}$ & 68.86$\ddel{+0.09}$ & 68.30$\ddel{+6.63}$ & 69.33$\ddel{+6.70}$ \\
\bottomrule
\end{tabular}
}
\caption{The results of LLaMA2-7B and LLaMA2-13B.}
\label{tbl:llama}
\end{table*}

\begin{table*}[htbp]
\tiny
\centering
\resizebox{0.85\textwidth}{!}{
\begin{tabular}{cccllll}
\toprule
& & &  \multicolumn{2}{c}{\textbf{OPT-6.7B}} & \multicolumn{2}{c}{\textbf{OPT-13B}}   \\
\cmidrule(lr){4-5}
\cmidrule(lr){6-7}
\multicolumn{2}{c}{\textbf{Datasets}} & \textbf{Decoding} & EM & F1 & EM & F1 \\
\midrule
\multirow{15}{*}{ NQ } & \multirow{5}{*}{ Total } & Regular & \cellcolor[gray]{0.9}19.74 &\cellcolor[gray]{0.9}26.25 &\cellcolor[gray]{0.9}21.11 &\cellcolor[gray]{0.9}30.14 \\
& & \text{SC} & 24.24$\ddel{+4.50}$ & 29.78$\ddel{+3.53}$ & 24.40$\ddel{+3.29}$ & 33.31$\ddel{+3.17}$  \\
& & \text{CD} & 22.90$\ddel{+3.16}$ & 34.48$\ddel{+8.23}$& 17.30$\db{-3.81}$ & 27.63$\db{-2.51}$\\
& & \text{CAD} & 29.15$\ddel{+9.41}$ & 40.16$\ddel{+13.91}$ & 24.76$\ddel{+3.65}$ & 36.37$\ddel{+6.23}$  \\
& & \ours &\textbf{30.07}$\ddel{+10.33}$ & \textbf{40.77}$\ddel{+14.52}$& \textbf{27.08}$\ddel{+5.97}$ & \textbf{38.87}$\ddel{+8.73}$\\
\cmidrule(lr){2-7}
& \multirow{5}{*}{ Conf. } & Regular &\cellcolor[gray]{0.9}19.79 &\cellcolor[gray]{0.9}26.24  &\cellcolor[gray]{0.9}20.72 &\cellcolor[gray]{0.9}29.33 \\
& & \text{SC} & 24.26$\ddel{+4.47}$ & 29.75$\ddel{+3.51}$ & 23.82$\ddel{+3.10}$ & 32.54$\ddel{+3.21}$ \\
& & \text{CD} & 22.96$\ddel{+3.17}$ & 34.54$\ddel{+8.30}$  & 17.25$\db{-3.47}$ & 27.58$\db{-1.75}$  \\
& & \text{CAD} & 29.21$\ddel{+9.42}$ & 40.19$\ddel{+13.95}$ & 24.55$\ddel{+3.83}$ & 36.19$\ddel{+6.86}$  \\
& & \ours & \textbf{30.13}$\ddel{+10.34}$ & \textbf{40.78}$\ddel{+14.54}$ & \textbf{26.80}$\ddel{+6.08}$ & \textbf{38.63}$\ddel{+9.30}$  \\
\cmidrule(lr){2-7}
& \multirow{5}{*}{ \makecell{Non-\\Conf.} } & Regular &\cellcolor[gray]{0.9}12.40 & \cellcolor[gray]{0.9}27.03  & \cellcolor[gray]{0.9}40.57 & \cellcolor[gray]{0.9}56.81  \\
& & \text{SC} & \textbf{21.79}$\ddel{+9.39}$ & 34.07$\ddel{+7.04}$ & \textbf{44.34}$\ddel{+3.77}$ & \textbf{60.36}$\ddel{+3.55}$ \\
& & \text{CD} & 12.51$\ddel{+0.11}$ & 25.23$\db{-1.80}$ & 18.87$\db{-21.70}$ & 29.18$\db{-27.63}$  \\
& & \text{CAD} & 20.83$\ddel{+8.43}$ & \textbf{36.05}$\ddel{+9.02}$ & 32.08$\db{-8.49}$ & 42.77$\db{-14.04}$  \\ 
& & \ours & 19.01$\ddel{+6.61}$ & 35.82$\ddel{+8.79}$  & 36.79$\db{-3.78}$ & 47.30$\db{-9.51}$  \\
\midrule
\multirow{15}{*}{ SQuAD } & \multirow{5}{*}{ Total } & Regular &\cellcolor[gray]{0.9}21.49 &\cellcolor[gray]{0.9}28.50  &\cellcolor[gray]{0.9}27.91 &\cellcolor[gray]{0.9}37.37 \\
& & \text{SC} & 23.64$\ddel{+2.15}$ & 30.97$\ddel{+2.47}$ & 30.13$\ddel{+2.22}$ & 40.08$\ddel{+2.71}$  \\
& & \text{CD} & 26.35$\ddel{+4.86}$ & 37.90$\ddel{+9.40}$ & 28.03$\ddel{+0.12}$ & 37.51$\ddel{+0.14}$  \\
& & \text{CAD} & 29.46$\ddel{+7.97}$ & 40.31$\ddel{+11.81}$  & 35.01$\ddel{+7.10}$ & 47.34$\ddel{+9.97}$ \\ 
& & \ours & \textbf{29.93}$\ddel{+8.44}$ & \textbf{40.47}$\ddel{+11.97}$ & \textbf{35.13}$\ddel{+7.22}$ & \textbf{47.48}$\ddel{+10.11}$  \\
\cmidrule(lr){2-7}
& \multirow{5}{*}{ Conf. } & Regular & \cellcolor[gray]{0.9}21.49 & \cellcolor[gray]{0.9}28.50 & \cellcolor[gray]{0.9}27.31 & \cellcolor[gray]{0.9}36.27\\
& & \text{SC} & 23.14$\ddel{+1.65}$ & 30.18$\ddel{+1.68}$  & 29.57$\ddel{+2.26}$ & 39.07$\ddel{+2.80}$  \\
& & \text{CD} & 26.33$\ddel{+4.84}$ & 37.61$\ddel{+9.11}$ & 27.42$\ddel{+0.11}$ & 36.38$\ddel{+0.11}$ \\
& & \text{CAD} & 29.32$\ddel{+7.83}$ & 39.97$\ddel{+11.47}$ & 34.79$\ddel{+7.48}$ & \textbf{46.93}$\ddel{+10.66}$  \\ 
& & \ours & \textbf{29.78}$\ddel{+8.29}$ & \textbf{40.13}$\ddel{+11.63}$ & \textbf{34.95}$\ddel{+7.64}$ & 46.84$\ddel{+10.57}$  \\
\cmidrule(lr){2-7}
& \multirow{5}{*}{ \makecell{Non-\\Conf.} } & Regular & \cellcolor[gray]{0.9}35.62 & \cellcolor[gray]{0.9}57.10& \cellcolor[gray]{0.9}40.30 & \cellcolor[gray]{0.9}60.28 \\
& & \text{SC} & \textbf{36.59}$\ddel{+0.97}$ & \textbf{56.04}$\db{-1.06}$ & \textbf{41.79}$\ddel{+1.49}$ & \textbf{61.17}$\ddel{+0.89}$  \\
& & \text{CD} & 26.90$\db{-8.72}$ & 49.05$\db{-8.05}$ & 40.67$\ddel{+0.37}$ & 60.98$\ddel{+0.70}$ \\
& & \text{CAD} & 34.93$\db{-0.69}$ & 53.44$\db{-3.66}$ & 40.41$\ddel{+0.11}$ & 58.93$\db{-1.35}$  \\
& & \ours & 35.69$\ddel{+0.07}$ & 53.71$\db{-3.39}$ & 38.81$\db{-1.49}$ & 60.88$\ddel{+0.60}$ \\
\midrule
\multirow{15}{*}{ StrategyQA } & \multirow{5}{*}{ Total} & Regular & \cellcolor[gray]{0.9}47.51 & \cellcolor[gray]{0.9}47.51 & \cellcolor[gray]{0.9}61.79 & \cellcolor[gray]{0.9}61.79  \\
& & \text{SC} & 46.64$\db{-0.87}$ & 46.64$\db{-0.87}$   & 60.57$\db{-1.22}$ & 60.57$\db{-1.22}$  \\
& & \text{CD} & 46.99$\db{-0.52}$ & 46.99$\db{-0.52}$ & 61.18$\db{-0.61}$ & 61.18$\db{-0.61}$ \\
& & \text{CAD} & 53.10$\ddel{+5.59}$ & 53.10$\ddel{+5.59}$ & 62.31$\ddel{+0.52}$ & 62.31$\ddel{+0.52}$  \\ 
& & \ours & \textbf{53.84}$\ddel{+6.33}$ & \textbf{53.84}$\ddel{+6.33}$ &  \textbf{64.33}$\ddel{+2.54}$ & \textbf{64.33}$\ddel{+2.54}$ \\
\cmidrule(lr){2-7}
& \multirow{5}{*}{ Conf. } & Regular & \cellcolor[gray]{0.9}47.86 & \cellcolor[gray]{0.9}47.86 & \cellcolor[gray]{0.9}61.48 & \cellcolor[gray]{0.9}61.48  \\
& & \text{SC} & 47.03$\db{-0.83}$ & 47.03$\db{-0.83}$  & 60.28$\db{-1.20}$ & 60.28$\db{-1.20}$  \\
& & \text{CD} & 47.26$\db{-0.60}$ & 47.26$\db{-0.60}$ & 60.86$\db{-0.62}$ & 60.86$\db{-0.62}$  \\
& & \text{CAD} &  54.21$\ddel{+6.35}$ & 54.21$\ddel{+6.35}$ & 61.97$\ddel{+0.49}$ & 61.97$\ddel{+0.49}$\\
& & \ours & \textbf{54.90}$\ddel{+7.04}$ & \textbf{54.90}$\ddel{+7.04}$  & \textbf{62.06}$\ddel{+0.58}$ & \textbf{62.06}$\ddel{+0.58}$   \\
\cmidrule(lr){2-7}
& \multirow{5}{*}{ \makecell{Non-\\Conf.} } & Regular & \cellcolor[gray]{0.9}40.87  & \cellcolor[gray]{0.9}40.87  & \cellcolor[gray]{0.9}82.35 & \cellcolor[gray]{0.9}82.35  \\
& & \text{SC} & 39.13$\db{-1.74}$ & 39.13$\db{-1.74}$& 79.41$\db{-2.94}$ & 79.41$\db{-2.94}$ \\
& & \text{CD} & \textbf{41.71}$\ddel{+0.84}$ & \textbf{41.71}$\ddel{+0.84}$ & 82.35$\ddel{+0.00}$ & 82.35$\ddel{+0.00}$  \\
& & \text{CAD} & 32.17$\db{-8.70}$ & 32.17$\db{-8.70}$  & \textbf{85.29}$\ddel{+2.94}$ & \textbf{85.29}$\ddel{+2.94}$ \\
& & \ours & 33.91$\db{-6.96}$ & 33.91$\db{-6.96}$  & 82.65$\ddel{+0.30}$ &  82.65$\ddel{+0.30}$\\
\midrule
\midrule
\multirow{5}{*}{ Counterfacts } & \multirow{5}{*}{ \makecell{Total\\(Conf.)}} & Regular & \cellcolor[gray]{0.9}18.15 &\cellcolor[gray]{0.9}19.38 &\cellcolor[gray]{0.9}19.55 &\cellcolor[gray]{0.9}20.75  \\
& & \text{SC} & 21.40$\ddel{+3.25}$ & 22.62$\ddel{+3.24}$  & 21.75$\ddel{+2.20}$ & 22.90$\ddel{+2.15}$ \\
& & \text{CD} & 38.16$\ddel{+20.01}$& 42.78$\ddel{+23.40}$ & 39.26$\ddel{+19.71}$& 42.89$\ddel{+22.14}$ \\
& & \text{CAD} &\textbf{40.10}$\ddel{+21.95}$ & \textbf{45.29}$\ddel{+25.91}$& \textbf{40.44}$\ddel{+20.89}$ & \textbf{47.46}$\ddel{+26.71}$  \\ 
& & \ours & 37.35$\ddel{+19.20}$ & 43.45$\ddel{+24.07}$  & 38.68$\ddel{+19.13}$ & 46.98$\ddel{+26.23}$  \\
\bottomrule
\end{tabular}
}
\caption{The results of OPT-6.7B and OPT-13B.}
\label{tbl:opt}
\end{table*}

\begin{table*}[htbp]
\tiny
\centering
\resizebox{0.85\textwidth}{!}{
\begin{tabular}{cccllllll}
\toprule
& & &  \multicolumn{2}{c}{\textbf{FLAN-T5-3B}} &\multicolumn{2}{c}{\textbf{FLAN-T5-11B}}  \\
\cmidrule(lr){4-5}
\cmidrule(lr){6-7}
\multicolumn{2}{c}{\textbf{Datasets}} & \textbf{Decoding} & EM & F1 & EM & F1 \\
\midrule
\multirow{15}{*}{ NQ } & \multirow{5}{*}{ Total } & Regular & \cellcolor[gray]{0.9}46.00 &\cellcolor[gray]{0.9}62.78 &\cellcolor[gray]{0.9}44.98 &\cellcolor[gray]{0.9}65.02 \\
& & \text{SC} & 46.14$\ddel{+0.14}$ & 62.51$\db{-0.27}$ & 44.28$\db{-0.70}$ & 65.71$\ddel{+0.69}$ \\
& & \text{CD} & 37.62$\db{-8.38}$ & 55.47$\db{-7.31}$  & 39.06$\db{-5.92}$ & 60.94$\db{-4.08}$ \\
& & \text{CAD} & 38.91$\db{-7.09}$ & 57.77$\db{-5.01}$ & 42.48$\db{-2.50}$ & 64.20$\db{-0.82}$ \\
& & \ours & \textbf{48.84}$\ddel{+2.84}$ & \textbf{64.45}$\ddel{+1.67}$ & \textbf{45.14}$\ddel{+0.16}$ & \textbf{65.98}$\ddel{+0.96}$ \\
\cmidrule(lr){2-7}
& \multirow{5}{*}{ Conf. } & Regular & \cellcolor[gray]{0.9}45.16 &\cellcolor[gray]{0.9}61.44&\cellcolor[gray]{0.9}42.35 &\cellcolor[gray]{0.9}62.69 \\
& & \text{SC} & 45.22$\ddel{+0.06}$ & 61.02$\db{-0.42}$& 39.77$\db{-2.58}$ & 62.25$\db{-0.44}$ \\
& & \text{CD} &  38.29$\db{-6.87}$ & 55.97$\db{-5.47}$ & 38.68$\db{-3.67}$ & 60.59$\db{-2.10}$ \\
& & \text{CAD} & 39.34$\db{-5.82}$ & 58.25$\db{-3.19}$  & 41.71$\db{-0.64}$ & 63.61$\ddel{+0.92}$ \\
& & \ours & \textbf{48.36}$\ddel{+3.20}$ & \textbf{63.98}$\ddel{+2.54}$ & \textbf{43.86}$\ddel{+1.51}$ & \textbf{64.73}$\ddel{+2.04}$ \\
\cmidrule(lr){2-7}
& \multirow{5}{*}{ \makecell{Non-\\Conf.} } & Regular & \cellcolor[gray]{0.9}52.20 &\cellcolor[gray]{0.9}72.65 &\cellcolor[gray]{0.9}60.81 &\cellcolor[gray]{0.9}79.09 \\
& & \text{SC} &52.26$\ddel{+0.06}$ & \textbf{73.49}$\ddel{+0.84}$  & 51.39$\db{-9.42}$ & 71.55$\db{-7.54}$ \\
& & \text{CD} & 33.40$\db{-18.80}$ & 52.33$\db{-20.32}$  & 41.40$\db{-19.41}$ & 63.03$\db{-16.06}$ \\
& & \text{CAD} &  35.68$\db{-16.52}$ & 54.22$\db{-18.43}$& 47.13$\db{-13.68}$ & 67.75$\db{-11.34}$ \\ 
& & \ours & \textbf{52.42}$\ddel{+0.22}$ & 67.87$\db{-4.78}$ & \textbf{52.87}$\db{-7.94}$ & \textbf{73.52}$\db{-5.57}$ \\
\midrule
\multirow{15}{*}{ SQuAD } & \multirow{5}{*}{ Total } & Regular & \cellcolor[gray]{0.9}71.20 &\cellcolor[gray]{0.9}83.53 &\cellcolor[gray]{0.9}66.63 &\cellcolor[gray]{0.9}80.88\\
& & \text{SC} & 70.90$\db{-0.30}$ & 83.28$\db{-0.25}$  & 67.96$\ddel{+1.33}$ & 81.51$\ddel{+0.63}$ \\
& & \text{CD} & 71.25$\ddel{+0.05}$ & 83.10$\db{-0.43}$  & 65.04$\db{-1.59}$ & 79.12$\db{-1.76}$\\
& & \text{CAD} & 68.62$\db{-2.58}$ & 81.88$\db{-1.65}$ & 68.88$\ddel{+2.25}$ & 81.91$\ddel{+1.03}$\\ 
& & \ours & \textbf{73.84}$\ddel{+2.64}$ & \textbf{84.99}$\ddel{+1.46}$ & \textbf{69.89}$\ddel{+3.26}$ & \textbf{82.59}$\ddel{+1.71}$ \\
\cmidrule(lr){2-7}
& \multirow{5}{*}{ Conf. } & Regular & \cellcolor[gray]{0.9}70.51 &  \cellcolor[gray]{0.9}83.09 &\cellcolor[gray]{0.9}65.34 &  \cellcolor[gray]{0.9}80.01\\
& & \text{SC} & 70.25$\db{-0.26}$ & 82.84$\db{-0.25}$ & 62.07$\db{-3.27}$ & 77.92$\db{-2.09}$ \\
& & \text{CD} & 71.31$\ddel{+0.80}$ & 83.17$\ddel{+0.08}$ & 64.57$\db{-0.77}$ & 78.93$\db{-1.08}$ \\
& & \text{CAD} & 68.64$\db{-1.87}$ & 81.92$\db{-1.17}$ & 68.43$\ddel{+3.09}$ & 81.73$\ddel{+1.72}$ \\ 
& & \ours & \textbf{73.51}$\ddel{+3.00}$ & \textbf{84.76}$\ddel{+1.67}$ & \textbf{69.20}$\ddel{+3.86}$ & \textbf{82.22}$\ddel{+2.21}$ \\
\cmidrule(lr){2-7}
& \multirow{5}{*}{ \makecell{Non-\\Conf.} } & Regular & \cellcolor[gray]{0.9}79.56 &\cellcolor[gray]{0.9}88.81 &\cellcolor[gray]{0.9}78.99 &\cellcolor[gray]{0.9}89.15 \\
& & \text{SC} &  \textbf{78.67}$\db{-0.89}$ & \textbf{88.61}$\db{-0.20}$ & \textbf{79.58}$\ddel{+0.59}$ & \textbf{89.18}$\ddel{+0.03}$ \\
& & \text{CD} & 70.60$\db{-8.96}$ & 82.26$\db{-6.55}$  & 69.57$\db{-9.42}$ & 80.89$\db{-8.26}$ \\
& & \text{CAD} & 68.37$\db{-11.19}$ & 81.42$\db{-7.39}$  & 73.19$\db{-5.80}$ & 83.62$\db{-5.53}$ \\
& & \ours & 77.78$\db{-1.78}$ & 87.76$\db{-1.05}$ & 76.45$\db{-2.54}$ & 86.14$\db{-3.01}$ \\
\midrule
\multirow{15}{*}{ StrategyQA } & \multirow{5}{*}{ Total} & Regular & \cellcolor[gray]{0.9}87.07 &\cellcolor[gray]{0.9}87.07 &\cellcolor[gray]{0.9}92.84 &\cellcolor[gray]{0.9}92.84 \\
& & \text{SC} &  86.81$\db{-0.26}$  & 86.81$\db{-0.26}$ &  92.58$\db{-0.26}$  & 92.58$\db{-0.26}$ \\
& & \text{CD} & \textbf{89.34}$\ddel{+2.27}$ & \textbf{89.34}$\ddel{+2.27}$ & 91.79$\db{-1.05}$ & 91.79$\db{-1.05}$ \\
& & \text{CAD} & 88.69$\ddel{+1.62}$ & 88.69$\ddel{+1.62}$ & 92.45$\db{-0.39}$ & 92.45$\db{-0.39}$ \\ 
& & \ours & 88.78$\ddel{+1.71}$ & 88.78$\ddel{+1.71}$& \textbf{92.89}$\ddel{+0.05}$ & \textbf{92.89}$\ddel{+0.05}$ \\
\cmidrule(lr){2-7}
& \multirow{5}{*}{ Conf. } & Regular & \cellcolor[gray]{0.9}69.41 &\cellcolor[gray]{0.9}69.41 &\cellcolor[gray]{0.9}83.44 &\cellcolor[gray]{0.9}83.44 \\
& & \text{SC} & 68.80$\db{-0.61}$ & 68.80$\db{-0.61}$& 83.18$\db{-0.26}$ & 83.18$\db{-0.26}$  \\
& & \text{CD} & \textbf{85.96}$\ddel{+16.55}$ & \textbf{85.96}$\ddel{+16.55}$ & \textbf{91.33}$\ddel{+7.89}$ & \textbf{91.33}$\ddel{+7.89}$  \\
& & \text{CAD} & 77.39$\ddel{+7.98}$ & 77.39$\ddel{+7.98}$& 87.06$\ddel{+3.62}$ & 87.06$\ddel{+3.62}$ \\
& & \ours & 77.36$\ddel{+7.95}$ & 77.36$\ddel{+7.95}$ & 87.39$\ddel{+3.95}$ & 87.39$\ddel{+3.95}$  \\
\cmidrule(lr){2-7}
& \multirow{5}{*}{ \makecell{Non-\\Conf.} } & Regular & \cellcolor[gray]{0.9}97.06 &\cellcolor[gray]{0.9}97.06  &\cellcolor[gray]{0.9}97.51 &\cellcolor[gray]{0.9}97.51 \\
& & \text{SC} & \textbf{96.69}$\db{-0.07}$ & \textbf{96.69}$\db{-0.07}$& \textbf{97.25}$\db{-0.26}$ &  \textbf{97.25}$\db{-0.26}$ \\
& & \text{CD} & 91.26$\db{-5.80}$ & 91.26$\db{-5.80}$  & 92.02$\db{-5.49}$ & 92.02$\db{-5.49}$ \\
& & \text{CAD} & 95.08$\db{-1.98}$ & 95.08$\db{-1.98}$ & 95.39$\db{-2.12}$ & 95.39$\db{-2.12}$ \\
& & \ours & 95.22$\db{-1.84}$ & 95.22$\db{-1.84}$ & 95.55$\db{-1.96}$ & 95.55$\db{-1.96}$ \\
\midrule
\midrule
\multirow{5}{*}{ Counterfacts } & \multirow{5}{*}{ \makecell{Total\\(Conf.)}} & Regular & \cellcolor[gray]{0.9}74.56 &\cellcolor[gray]{0.9}75.73 &\cellcolor[gray]{0.9}71.79 &\cellcolor[gray]{0.9}74.82 \\
& & \text{SC} & 74.58$\ddel{+0.02}$ & 75.64$\db{-0.09}$  & 72.60$\ddel{+0.81}$ & 74.30$\db{-0.52}$ \\
& & \text{CD} & 74.76$\ddel{+0.20}$ & 77.30$\ddel{+1.57}$ & 70.31$\db{-1.48}$ & 75.67$\ddel{+0.85}$ \\
& & \text{CAD} & 68.23$\db{-6.33}$ & 74.17$\db{-1.56}$ & 67.39$\db{-4.40}$ & 74.42$\db{-0.40}$  \\ 
& & \ours & \textbf{77.60}$\ddel{+3.04}$ & \textbf{78.97}$\ddel{+3.24}$ & \textbf{75.29}$\ddel{+3.50}$ & \textbf{78.37}$\ddel{+3.55}$  \\
\bottomrule
\end{tabular}
}
\caption{The results of FLAN-T5-3B and FLAN-T5-11B.}
\label{tbl:flant5}
\end{table*}

\end{document}